\documentclass[runningheads]{llncs}
\usepackage{graphicx}
\usepackage{amsmath,amsfonts,bm}

\def\eqref#1{equation~\ref{#1}}

\def\1{\bm{1}}

\DeclareMathAlphabet{\mathsfit}{\encodingdefault}{\sfdefault}{m}{sl}
\SetMathAlphabet{\mathsfit}{bold}{\encodingdefault}{\sfdefault}{bx}{n}

\def\gN{{\mathcal{N}}}

\def\sR{{\mathbb{R}}}

\newcommand{\E}{\mathbb{E}}

\DeclareMathOperator*{\argmax}{arg\,max}

\usepackage{booktabs}
\usepackage{graphicx}
\usepackage{hyperref}
\usepackage{url}
\usepackage{xcolor}
\usepackage{algorithm}
\usepackage[noend]{algpseudocode}
\usepackage{wrapfig}
\usepackage{amssymb,amsmath} 
\usepackage{wrapfig}
\usepackage{bbm}
\usepackage{adjustbox}

\newcommand{\madam}{\textsc{MAdam}}
\newcommand{\lamadam}{\textsc{LaMAdam}}
\newcommand{\adam}{\textsc{Adam}}
\newcommand{\laprop}{\textsc{LaProp}}
\newcommand{\lamb}{\textsc{Lamb}}
\newcommand{\signsgd}{\textsc{signSGD}}

\newcommand{\rmsprop}{\textsc{RMSprop}}

\newcommand{\amsgrad}{\textsc{AMSGrad}}
\newcommand{\adabound}{\textsc{AdaBound}}
\newcommand{\yogi}{\textsc{Yogi}}

\begin{document}
\title{MaxVA: Fast Adaptation of Step Sizes by Maximizing Observed Variance of Gradients}
\titlerunning{MaxVA: Fast Adaptation of Step Sizes}
%
\author{Chen Zhu\inst{1} \and 
        Yu Cheng\inst{2} \and 
        Zhe Gan \inst{2} \and
        Furong Huang \inst{1} \and \\
        Jingjing Liu \inst{3} \and
        Tom Goldstein \inst{1}}
\authorrunning{Zhu et al.}
%
\institute{University of Maryland \email{\{chenzhu,furongh,tomg\}@umd.edu} \and
Microsoft
  \email{\{yu.cheng,zhe.gan,jingjl\}@microsoft.com} \and
  Tsinghua University \email{jjliu@air.tsinghua.edu.cn}
  }
\maketitle              
\begin{abstract}
Adaptive gradient methods such as \rmsprop~and \adam~use exponential moving estimate of the squared gradient to compute adaptive step sizes, achieving better convergence than SGD in face of noisy objectives. However, \adam~can have undesirable convergence behaviors due to unstable or extreme adaptive learning rates. Methods such as \amsgrad~and \adabound~have been proposed to stabilize the adaptive learning rates of \adam~in the later stage of training, but they do not outperform \adam~in some practical tasks such as training Transformers~\cite{transformer}. In this paper, we propose an adaptive learning rate principle, in which the running mean of squared gradient in \adam~is replaced by a weighted mean, with weights chosen to maximize the estimated variance of each coordinate. 
This results in a faster adaptation to the local gradient variance, which leads to more desirable empirical convergence behaviors than~\adam. 
We prove the proposed algorithm converges under mild assumptions for nonconvex stochastic optimization problems, and demonstrate the improved efficacy of our adaptive averaging approach on machine translation, natural language understanding and large-batch pretraining of BERT. The code is available at \url{https://github.com/zhuchen03/MaxVA}.
\end{abstract}
%
%

\section{Introduction}
Stochastic Gradient Descent (SGD) and its variants are commonly used for training deep neural networks because of their effectiveness and efficiency. In their simplest form, gradient methods train a network by iteratively moving each parameter in the direction of the negative gradient (or the running average of gradients) of the loss function on a randomly sampled mini-batch of training data. A scalar learning rate is also applied to control the size of the update.
In contrast, {\em adaptive} stochastic gradient methods use coordinate-specific learning rates, which are inversely proportional to the square root of the running mean of squared gradients \cite{Tieleman2012rmsprop,adagrad,adam}. Such methods are proposed to improve the stability of SGD on non-stationary problems, and have achieved success in different fields across Speech, Computer Vision, and Natural Language Processing.

Large pretrained Transformer-based language models have achieved remarkable successes in various language tasks~\cite{devlin2019bert,liu2019roberta,lan2019albert,raffel2019t5,brown2020gpt3}. 
The original Transformer architecture (Post-LN Transformers) often demonstrates better performance than its Pre-LN variant~\cite{liu2020understanding}, but its gradient has high variance during training. A warmup learning rate schedule or small initial adaptive learning rates~\cite{liu2019radam} are required for its convergence. 
\cite{zhang2019longtail} shows that SGD fails to train Transformers without gradient clipping, and adaptivity is important for stabilizing optimization under the heavy-tailed noise in Transformer's gradients.
This indicates that the strategy of \adabound~\cite{luo2019adabound}, which is to transition from \adam~into SGD, may fail on Post-LN Transformers (see~Appendix \ref{sec:adabound} for instance).
However, the adaptive learning rate of Adam can be unstable in the later stage of training, and such instability sometimes leads to sub-optimal solutions or even non-convergent behavior on some simple problems~\cite{reddi2019amsgrad,luo2019adabound}.
\amsgrad~\cite{reddi2019amsgrad} was proposed to deal with this issue by computing the adaptive learning rate with an update rule that guarantees monotonically decaying adaptive learning rates for each coordinate, but to our knowledge, it has not been widely deployed to enhance \adam~for training Transformer-based language models. 

In this work, we explore a different approach to improving the stability of adaptive learning rates.
We propose \emph{Maximum Variation Averaging} (MaxVA), which computes the running average of squared gradients using dynamic, rather than constant, coordinate-wise weights. These weights are chosen so that the estimated variance of gradients is maximized, to enable a faster adaptation to the changing variance of gradients. 
The MaxVA weights for maximizing this variance have a simple closed-form solution that requires little storage or computational cost. 
With MaxVA, the adaptive optimizer 
1) takes a smaller step size when abnormally large gradient is present, to improve stability; 
2) takes a larger step size when abnormally small gradient is prevent, to avoid spurious minima and achieve better generalization~\cite{li2018visualizing}; 
3) takes a steady step size when gradients are stable and within estimated deviation, to ensure convergence~\cite{reddi2019amsgrad}.
In the large-batch setting of BERT pretraining,  where the total number of iterations is sharply reduced and a faster adaptation in each step is more important, MaxVA achieves faster convergence and obtain models with better test performance on downstream tasks than both \adam~and \lamb~\cite{you2020lamb}.
Extensive experiments on both synthetic and practical datasets demonstrate that MaxVA leads to an improved adaptability and stability for \adam, yielding better test set performance than \adam~on a variety of tasks. 
We also prove MaxVA converges under mild assumptions in the nonconvex stochastic optimization setting.

\section{Preliminary and Definitions}
By default, all vector-vector operators are element-wise in the following sections.
Let ${\theta}\in \sR^d$ be the parameters of the network to be trained, $\ell({x};{\theta})$ is the loss of the model with parameters ${\theta}$ evaluated at ${x}$.
Our goal is to minimize the expected risk on the data distribution defined as:
\begin{equation}\label{eq:obj}
    f(\theta)=\mathbb{E}_{{x} \sim \mathcal{D}} \left[\ell({x};{\theta})\right].
\end{equation}

In most deep learning problems, only a finite number of potentially noisy samples can be used to approximate Eq.~\ref{eq:obj}, and the gradients are computed on randomly sampled minibatches during training. 
Stochastic regularizations such as Dropout~\cite{srivastava2014dropout} are commonly used for training Transformer-based language models~\cite{transformer,Zhu2020FreeLB}, which further adds to the randomness of the gradients.
Thus, it is important to design optimizers that tolerate noisy gradients. 
\adam~\cite{adam} is an effective optimizer that adapts to such noisy gradients.
It keeps exponential moving averages ${m}_t$ and ${v}_t$ of past gradients $g_1,...,g_{t-1}$, defined as: 
\begin{equation*}\label{eq:adam}
\begin{split}
    {\tilde{m}}_t &=\alpha {\tilde{m}}_{t-1} + (1-\alpha) {g}_t,
    \quad 
    {m_t}=\frac{{\tilde{m}}_t}{1-\alpha^{t+1}},\\
    {\tilde{v}}_t &=\beta {\tilde{v}}_{t-1} + (1-\beta) {g}_t^2,
    \quad 
    {v}_t=\frac{{\tilde{v}}_t}{1-\beta^{t+1}},
\end{split}
\end{equation*}
where $\alpha, \beta\in [0,1]$, $g_t=\nabla_\theta \ell (x_t;\theta_t)$ is the gradient of the $t$-th minibatch $x_t$, ${\tilde{m}}_0={\tilde{v}}_0={0}$, and $m_t, v_t$ corrects this zero-initialization bias of $\tilde{m}_t, \tilde{v}_t$~\cite{adam}. \adam~updates the parameters with the estimated moments as ${\theta}_{t+1}={\theta}_t - \eta_t \frac{{m}_t}{\sqrt{{v}_t}+\epsilon}$, where $\epsilon>0$ is a small constant for numerical stability.

If we assume that the distribution of the stochastic gradient is constant within the effective horizon of the running average, then $m_t$ and $v_t$ will be estimates of the first and second moments of the gradient $g_t$~\cite{balles2018dissecting}. 
Same as other adaptive methods such as \adam~and the recently proposed AdaBelief~\cite{zhuang2020adabelief}, we adopt this assumption throughout training.
With this assumption, at time $t$, we assume $\E[{m}_t]\approx \nabla f_t$, $\E[{v}_t] \approx \nabla f_t^2 + \sigma_t^2$, where $\sigma_t^2$ is the variance of $g_t$. 
\adam, \rmsprop~and other variants that divide the update steps by $\sqrt{{v}_t}$ can be seen as adapting to the gradient variance under this assumption when $m_t$ is small. 
These adaptive methods take smaller step sizes when the estimated variance ${\sigma}_t^2=v_t-m_t^2$ is high. 
Higher local gradient variance indicates higher local curvature, and vice versa. 
In certain quadratic approximations to the loss function, this variance is proportional to the curvature~\cite{schaul2013pesky} (Eq.~\ref{eq:nqm_grad} of our paper).
Therefore, like a diagonal approximation to Newton's method, such adaptative learning rates adapt to the curvature and can accelerate the convergence of first-order methods.

However, the adaptive learning rate $\eta_t/(\sqrt{{v}_t}+\epsilon)$ of \adam~and \rmsprop~can take extreme values, causing convergence to undesirable solutions~\cite{wilson2017marginal,chen2018convergence}.
\cite{reddi2019amsgrad} gave one such counter example where gradients in the correct direction are large but occur at a low frequency, and \adam~converges to the solution of maximum regret.
They solve this issue by keeping track of the maximum ${v}_t$ for each coordinate throughout training with a new variable $\hat{{v}}_t$, and replace the adaptive learning rate with $\eta_t/\sqrt{\hat{{v}}_t}$ to enforce monotonically descreasing learning rates. 
Extremely small adaptive learning rates can also cause undesirable convergence behavior, as demonstrated by a counter example from~\cite{luo2019adabound}. 

\section{Maximizing the Variance of Running Estimations}
\begin{figure}[htbp!]
\centering
\includegraphics[width=0.6\linewidth]{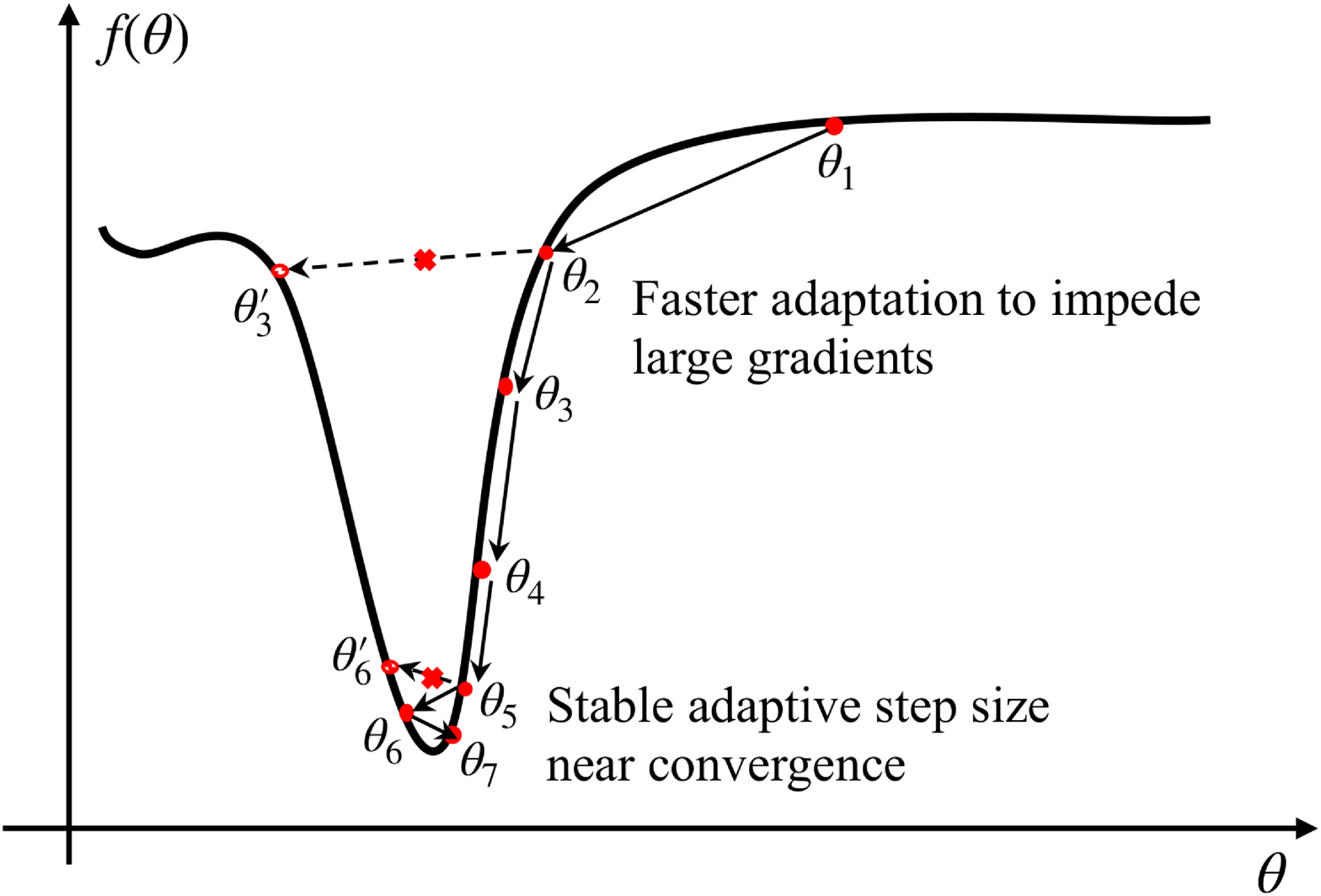}
\caption{\small{An illustrative example of MaxVA. It adopts a smaller adaptive learning rate when an abnormally large gradient appears at $\theta_2$ by choosing a smaller $\beta$, to prevent overshooting to $\theta_3'$. When the gradient is more stable near convergence at $\theta_5$, it uses a larger $\beta$ for the slowest change in adaptive learning rate, to prevent overshooting to $\theta_6'$.}}
\label{fig:illu}
\end{figure}
\textbf{Motivation.} We propose to mitigate the undesirable convergence issue of \adam{} by changing the constant running average coefficient $\beta$ for the second moment into an adaptive one. 
The idea is to allow $\beta_t$ to adopt the value that maximizes the estimated variance of the gradient at each iteration $t$. 
As a result, our method will assign a higher coefficient $(1- \beta_t)$ to $g_t$ when it deviates too much from the estimated mean, resulting in a smaller step size when the gradient is too large to avoid overshooting, and a larger step size when the gradient is abnormally small to avoid spurious local minima. By contrast, if the gradient is stable and close to the estimated mean, which often happens near a flat minimum, our method will assign minimum coefficient to $g_t$ to maintain the value of $v_t$ and take steady steps towards the minimum. 
Therefore, our method can use $\beta_t$ as the adaptive running average coefficient to take steps that are cautious enough to avoid instability and spurious minima but aggressive enough to make progress. An illustrative example is given in Figure~\ref{fig:illu}.

\textbf{Maximum Variation Averaging.} Formally, we estimate the variance of the gradient at each coordinate by keeping track of the zeroth, first, and second moments of the gradient as functions of the adaptive running average coefficient $\beta_t$, denoted as $w_t (\beta_t),~\tilde{{u}}_t(\beta_{t})$ and $\tilde{v}_t(\beta_t)$, respectively: 
\begin{align}
    w_t(\beta_t) &= \beta_t w_{t-1}(\beta_{t-1}) + (1-\beta_t), \label{eq:zeroth}\\
    \tilde{{u}}_t(\beta_t) &= \beta_t \tilde{{u}}_{t-1}(\beta_{t-1}) + (1-\beta_t){g}_t, \label{eq:first} \\
    \tilde{{v}}_t(\beta_t) &= \beta_t \tilde{{v}}_{t-1}(\beta_{t-1}) + (1-\beta_t){g}_t^2.\label{eq:second} 
\end{align}
The zeroth moment $w_t(\beta_t)$ is used to normalize $\tilde{u}_t(\beta_t)$ and $\tilde{v}_t(\beta_t)$ to achieve bias-corrected estimates $u_t(\beta_t)=\tilde{u}_t(\beta_t)/w_t(\beta_t)$ and $v_t(\beta_t)=\tilde{v}_t(\beta_t)/w_t(\beta_t)$ for the first and second moments, so that the estimates are not biased towards zero ($\tilde{m}_0=\tilde{v}_0=0$)~\cite{adam}.

Under our assumptions, the bias-corrected local estimate of the gradient variance is ${{\sigma}}_t^2 = \tilde{{v}}_t(\beta_t) / w_t(\beta_t) - [\tilde{{u}}_t(\beta_t) / w_t(\beta_t)]^2$. 
Taking the arg max for ${{\sigma}}_t^2$, we find the $\beta_t$ that achieves the maximal variance for each coordinate $i$:
\begin{equation}
\beta_{t,i} = \argmax_{\beta} {{\sigma}}_{t, i}^2 = \argmax_{\beta} {v}_{t,i}(\beta)  - [{u}_{t, i}(\beta)]^2. \label{eq:max-form}
\end{equation}
We call our approach to finding adaptive running average coefficient $\beta_t$ \emph{Maximum Variation Averaging} (MaxVA). 
We plug MaxVA into \adam~and its variant \laprop~\cite{ziyin2020laprop}, which results in two novel algorithms, \madam~and \lamadam, listed in Algorithm~\ref{alg:madam} and Algorithm~\ref{alg:LaMadam} (in the Appendix).
Different from \adam, \laprop~uses $v_t$ to normalize the gradients before taking the running average, which results in higher empirical stability under various hyperparameters.
Note, we only apply the adaptive $\beta_t$ to the \textit{second} moment $u_t(\beta_t)$ used for scaling the learning rate; $m_t$ is still an exponential moving average \textit{with a constant coefficient $\alpha$} of the gradient for \madam~or the normalized gradient for \lamadam.

\begin{algorithm}[H]\small
	\caption{\madam}
	\label{alg:madam}
	\begin{algorithmic}[1]
		\State {\bfseries Input:} Learning rate $\{\eta_t\}_{t=1}^T$, parameter $0 < \alpha < 1$, $0<\underline{\beta}<\bar{\beta}<1$, $\epsilon > 0$
		\State Set $\tilde{m}_{0}=\tilde{u}_0=\tilde{v}_{0}=w_0 = 0$
		\For{$t=1$ {\bfseries to} $T$}
		\State Draw samples $S_t$ from training set
        \State Compute $g_t = \frac{1}{|S_t|} \sum_{x_k \in \mathcal{S}_t}\nabla \ell(x_k; \theta_t)$
        \State $\tilde{m}_{t} = \alpha \tilde{m}_{t-1} + (1 - \alpha) g_{t}$
        \State $\tilde{\beta}_t={\arg\max}_\beta v_t(\beta)-u_t^2(\beta)$ \Comment{see Eq \ref{eq:beta_sol}}
        \State $\beta_t = \max(\underline{\beta}, \min(\bar{\beta}, \tilde{\beta}_t))$
        \State $\tilde{u}_t=\beta_t \tilde{u}_{t-1}+(1-\beta_t)g_t$
        \State $\tilde{v}_t=\beta_t \tilde{v}_{t-1}+(1-\beta_t)g^2_t$
        \State $w_t=\beta_t w_{t-1}+(1-\beta_t)$
		\State $\theta_{t} = \theta_{t-1} - \eta_t\frac{\sqrt{w_t}}{1-\alpha^t} \frac{ \tilde{m}_t}{\sqrt{\tilde{v}_t}+\epsilon} $
		\EndFor
	\end{algorithmic}
\end{algorithm}

\textbf{Finding $\beta_t$ via a Closed-form Solution.} The maximization for $\beta_t$ in Eq. \ref{eq:max-form} is quadratic and has a relatively simple closed-form solution that produces maximal ${\sigma}_t^2$ for each coordinate:
\begin{equation}\label{eq:beta_sol}
\begin{split}
    \beta_t &= \frac{\Delta g_t^2 + {\sigma}_{t-1}^2}{w_{t-1}(\Delta g_t^2 - {\sigma}_{t-1}^2)  + \Delta g_t^2 + {\sigma}_{t-1}^2},\\
\end{split}
\end{equation}
where all variables are vectors and all the operations are elementwise, $\Delta g_t=(g_t- u_{t-1})$ is the deviation of the gradient $g_t$ from the estimated mean $u_{t-1}$, ${\sigma}_{t-1}^2 =v_{t-1} - u_{t-1}^2$ is the estimated variance, and we have abbreviated $u_{t-1}(\beta_{t-1})$, $v_{t-1}(\beta_{t-1})$ and $w_{t-1}(\beta_{t-1})$ into ${u}_{t-1}, {v}_{t-1}$ and $w_{t-1}$. We use this abbreviation in the following sections, and defer the derivation of Eq.~\ref{eq:beta_sol} to Appendix~\ref{appendix:closed_form}.

\textbf{Implementation Notes.} 
We apply MaxVA in every step except for the first step, where the gradient variance one can observe is zero.
So for Algorithm~\ref{alg:madam} and Algorithm~\ref{alg:LaMadam} we define:
\begin{equation}
\tilde{u}_1= (1 - \beta_1) g_1, \tilde{v}_1= (1-\beta_1) g_1^2, w_1=1 - \beta_1.
\end{equation}
The coefficient $\beta_1$ for $t=1$ is set to a constant that is the same as typical values for \adam.
To obtain a valid running average, we clip $\beta_t$ so that $\underline{\beta} \le \beta_t \le \bar{\beta}$, where the typical values are $\underline{\beta}=0.5, 0.98 \le \bar{\beta}\le 1$. 
For convenience, we set $\beta_1=\bar{\beta}$ by default.
For $t>1,~\text{since }0<\beta_t\le 1$, $w_t$ will monotonically increase from $(1-\beta_1)$ to 1. 
Before clipping, for any $g_t, u_{t-1}, v_{t-1}$ satisfying $v_{t-1}-u_{t-1}^2 > 0$ in Eq.~\ref{eq:beta_sol}, we have $ \beta_t \in [1/(1+w_{t-1}), 1/(1-w_{t-1})]$. 
As a result, the lower bound that we use ($\underline{\beta}=0.5$) is tight and does not really change the value of $\beta_t$, and as $t \rightarrow \infty$, $w_t\rightarrow 1$ and $\beta_t\in [0.5, \infty]$. 
We have a special case at $t=2$, where $\beta_t$ is a constant $1/(2-\beta_1)$.

In practice, we also add a small coefficient $\delta>0$ to the denominator of Eq.~\ref{eq:beta_sol} to prevent division by zero, which will have negligible effect on the value of $\beta_t$ and does not violate the maximum variation objective (Eq.~\ref{eq:max-form}). 
All the derivations for these conclusions are deferred to Appendix~\ref{sec:div}.

\textbf{Effect of Maximum Variation Averaging.} 
By definition, we have $\sigma_{t-1}^2 \ge  0$, but in most cases $\sigma_{t-1}^2 > 0$. 
When $\sigma_{t-1}^2 > 0$, we define a new variable $R_t=\Delta g_t^2 / \sigma_{t-1}^2 $, which represents the degree of deviation of gradient $g_t$ from the current estimated average. Then, we can rewrite: 
\begin{equation}\label{eq:beta_t}
    \beta_t = \frac{R_t+1}{(1+w_t)R_t + 1 - w_t}.
\end{equation}
From Eq.~\ref{eq:beta_t}, we can see $\beta_t$ monotonically decreases from $1/(1-w_t)$ to $1/(1+w_t)$ as $R_t$ increases from 0 to $\infty$, and equals to 1 when $R_t=1$.
As a result, for each coordinate, if $R_t\gg 1$, $g_t$ deviates much more than ${\sigma}_{t-1}$ from $u_{t-1}$, and MaxVA will find a smaller $\beta_t$ and therefore a higher weight $(1-\beta_t)$ on $g_t^2$ to adapt to the change faster. 
This helps to avoid overshooting when abnormally large gradient is present (see Figure~\ref{fig:illu}), and avoids spurious sharp local minima where gradients are abnormally small.
With a faster response to abnormal gradients, MaxVA is better at handling the heavy-tailed distribution of gradients in the process of training Transformers~\cite{zhang2019longtail}.
In practice, $v_t$ tends to be larger than \adam/\laprop~using a constant $\bar{\beta}$, but as we will show in the experiments, using a larger learning rate counters such an effect and achieves better results.

On the other hand, if $R_t< 1$, or the deviation of the gradient $g_t$ from the current running mean ${u}_{t-1}$ is within the estimated standard deviation ${\sigma}_{t-1}$, we will use $\bar{\beta}$ to update $\tilde{v}_t$, which is the smallest change we allow for $\tilde{v}_t$. This tends to happen in the later phase of training, where the gradient variance decreases. MaxVA will adopt a steady step towards convergence by finding the slowest rate to update $\tilde{v}_{t}$.
This allows large values of $\tilde{v}_t$ to last for a longer horizon even compared with setting $\beta_t$ to a constant $\bar{\beta}$ on the same sequence, since we have assigned more mass to large gradients, which can be seen as an adaptive version of \amsgrad. 
Note that MaxVA and \amsgrad~can be complementary approaches if applied together, which we have found helpful for Image Classification on CIFAR10/100.

\textbf{Convergence Analysis.}\, We prove the convergence of MaxVA in the nonconvex stochastic optimization setting. 
For the sake of simplicity, we analyze the case where $\alpha=0$, which is effectively applying MaxVA to \rmsprop. We leave the analysis for $\alpha\neq 0$ for future research. 
We assume the function $\ell$ is $L$-smooth in $\theta$, i.e., there exists a constant $L$ such that for all $\theta_1, \theta_2\in \mathbb{R}^d, x\in \mathcal{X}$, 
\begin{equation}
    \lVert \nabla_\theta \ell(x;\theta_1) - \nabla_\theta \ell(x;\theta_2) \rVert \le L\lVert \theta_1 - \theta_2 \rVert. 
\end{equation}
This automatically implies that $f(\theta)=\mathbb{E}[\ell(x;\theta)]$ is $L$-smooth. 
Such a smoothness assumption holds for networks with smooth activation functions, e.g., Transformers that use the GELU activation~\cite{hendrycks2016gaussian}.
We also need to assume function $\ell$ has bounded gradient, i.e., $\lVert \nabla_\theta \ell(x;\theta) \rVert_\infty \le G$ for all $ \theta \in \mathbb{R}^d, x\in \mathcal{X}$.
As typically used in the analysis of stochastic first-order methods~\cite{zaheer2018yogi,ghadimi2013stochastic}, 
we assume the stochastic gradient has bounded variance: $\mathbb{E}[ [\nabla_{\theta} \ell(x;\theta)]_i - [\nabla_{\theta} f(\theta)]_i ]^2 \le \sigma^2$ for all $ \theta\in \mathbb{R}^d$.
Further, we assume the batch size increases with time as $b_t=t$, which is also adopted in the analysis of \signsgd~\cite{bernstein2018signsgd}, and holds in our large batch experiments. 
Theorem~\ref{thm:convergence} gives a ``worst-case" convergence rate of MaxVA to a stationary point under these assumptions, where the dependence of $\beta_t$ on $g_t$ is ignored and we only consider the worst-case of $\beta_t$ in each step. 
The proof is given in Appendix~\ref{sec:convergence_proof}.
\begin{theorem}\label{thm:convergence}
Define $w_0=1$. Let $\eta_t=\eta$ and $b_t=t$ for all $t\in [T]$. Furthermore, we assume $\epsilon, \underline{\beta}, \bar{\beta}, \eta$ are chosen such that $ \eta\le \frac{\epsilon}{2L}\text{, } 1-\underline{\beta}\le \frac{\epsilon^2}{16G^2}\text{, and } \bar{\beta}\le 2\underline{\beta}$. Then for $\theta_t$ generated using \madam, we have the following bound:
\begin{equation}
    \mathbb{E} \lVert \nabla f(\theta_a) \rVert^2 \le O\left( \frac{f(\theta_1) - f(\theta^*)}{\eta T} + \frac{2\sigma dG}{\epsilon \sqrt{T}} \right),
\end{equation}
where $\theta^*$ is an optimal solution to minimize the objective in Eq.~\ref{eq:obj}, and $\theta_a$ is an iterate uniformly randomly chosen from $\{\theta_1,...,\theta_T\}$.
\end{theorem}

\section{Experiments on Synthetic Data}
For a quantitative control of the stochasticity and data distribution, which affects the difficulty of the problem and the efficacy of the optimizers, we compare \madam~and the baselines in two sets of synthetic data, and demonstrate the efficacy of MaxVA with statistical significance on a large number of instances.
The first dataset simulates prevalent machine learning settings, where mini-batch stochastic gradient methods are applied on a finite set of samples, on which we show \madam~fixes the nonconvergence issue of \adam~and achieves faster convergence rate than \amsgrad.
The second dataset evaluates the algorithms under different curvatures and gradient noise levels, where we show \madam~achieves both lower loss and variance than fine-tuned \adam~at convergence.

\subsection{Convergence with Stochastic Gradients} 
\begin{figure*}[h]\label{fig:regret}
\vspace{-1em}
\centering
\includegraphics[width=.32\linewidth]{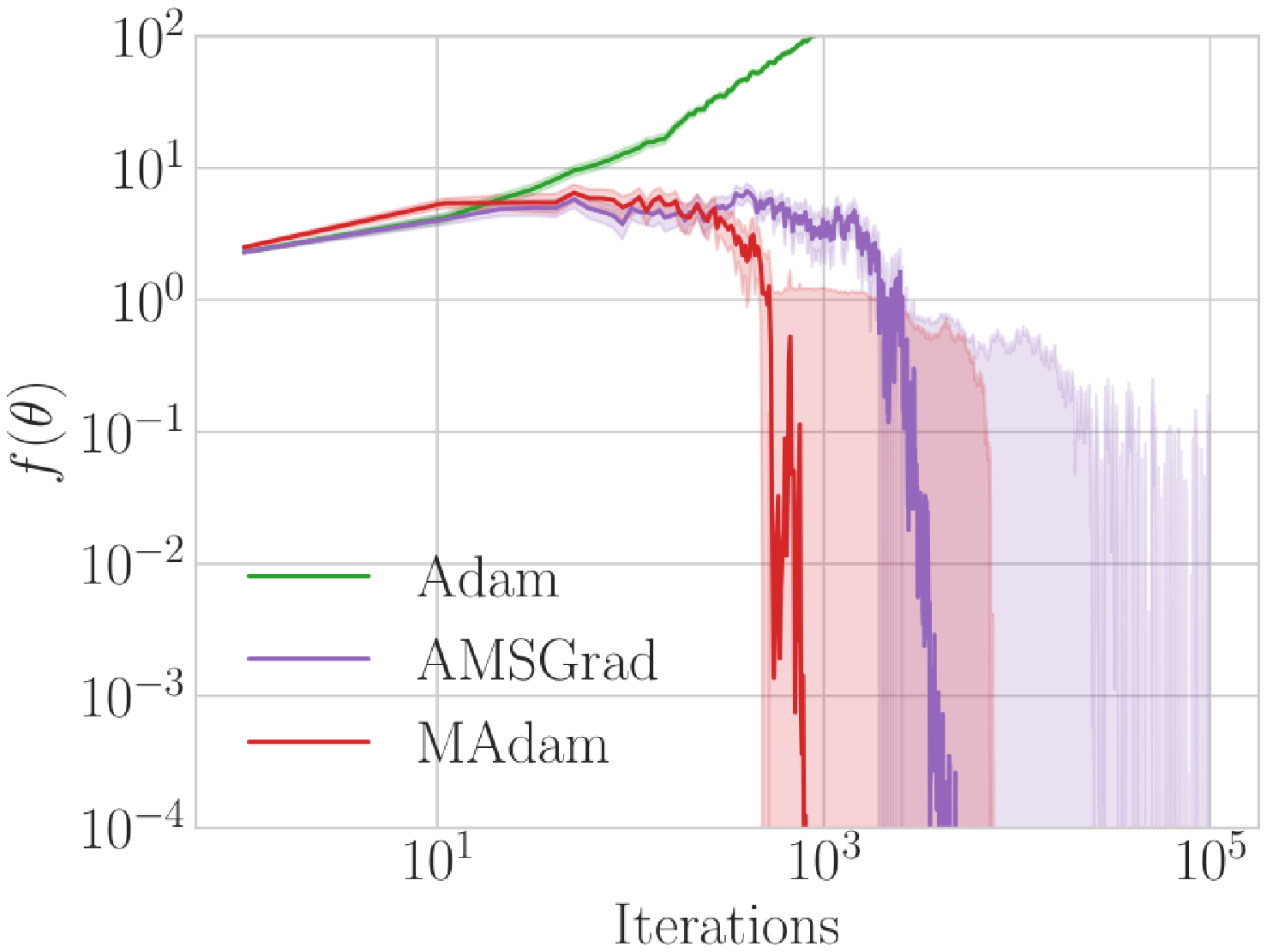}
\includegraphics[width=.32\linewidth]{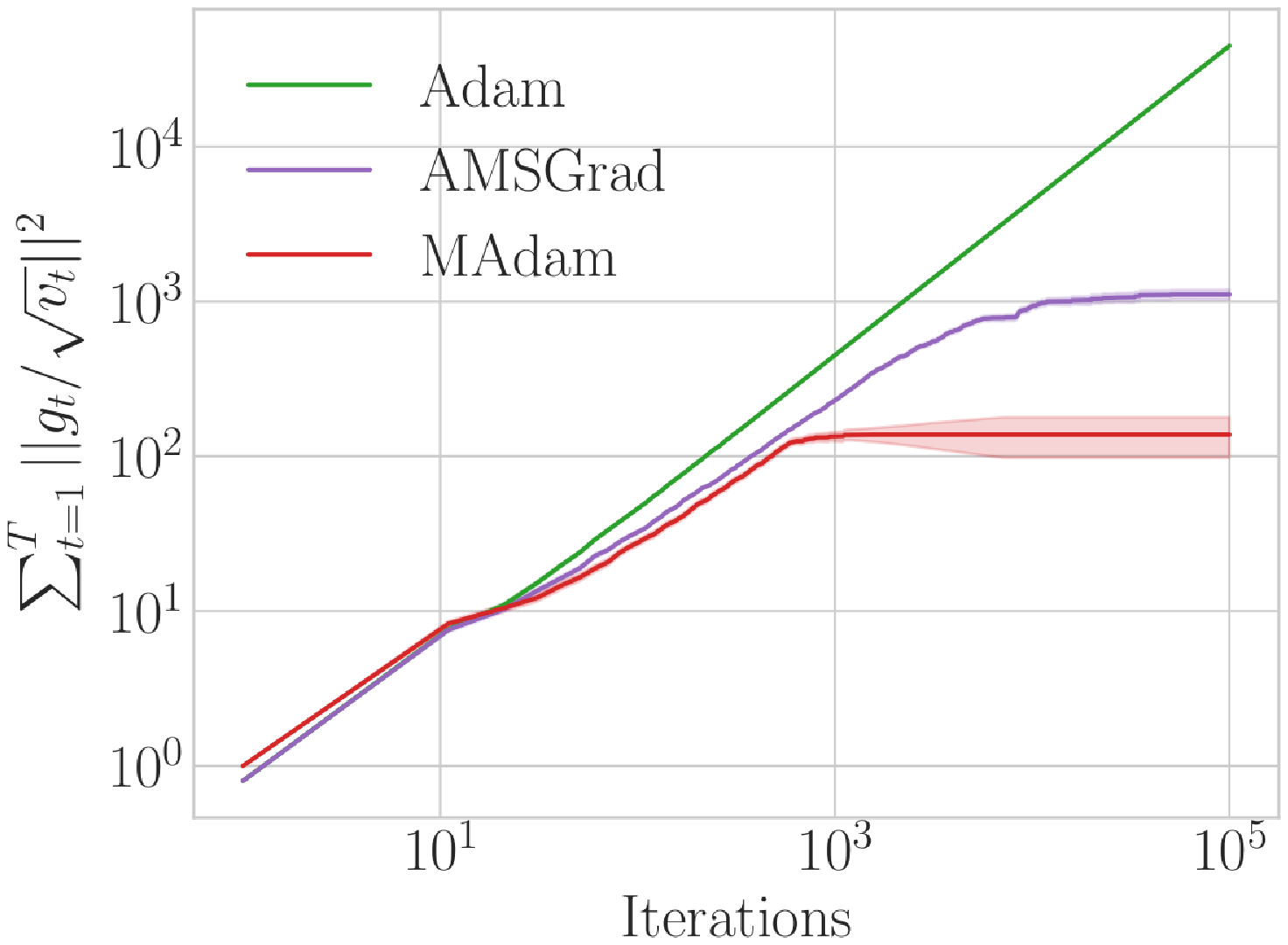}
\includegraphics[width=.32\linewidth]{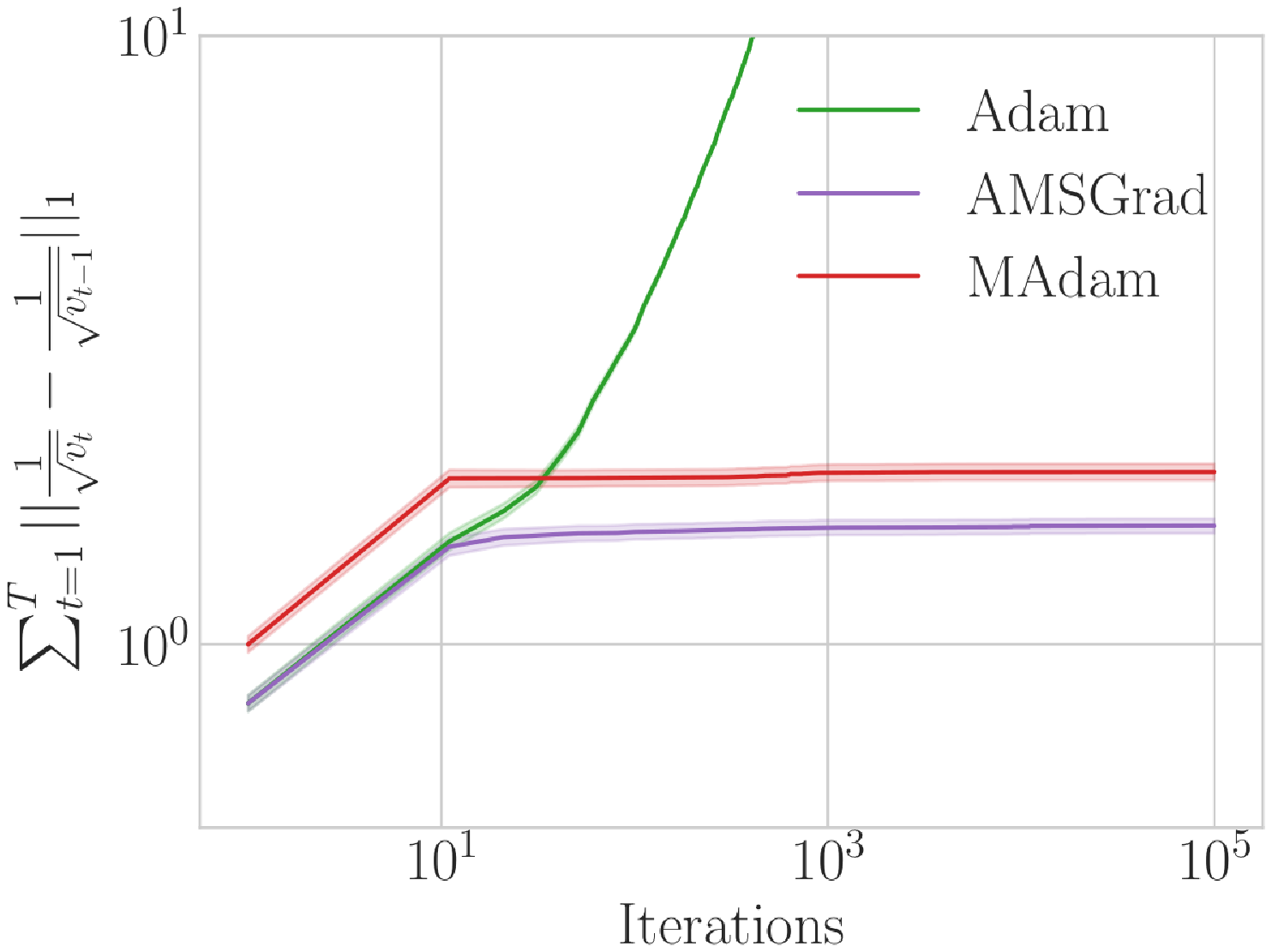}
\vspace{-1em}
\caption{\footnotesize Median and standard error (100 runs) of objective value ($f(\theta)$), accumulated update size ($\sum_{t=1}^{T}|| g_t/\sqrt{v_t}||^2$) and total change in adaptive learning rate ($\sum_{t=1}^{T}|| \frac{1}{\sqrt{v_t}}-\frac{1}{\sqrt{v_{t-1}}} ||_1$) for \adam, \amsgrad, \madam~on the problem in Eq.~\ref{eq:toy_converge}. }
\end{figure*}
Since MaxVA maximizes the variance and the gradient converges to zero in most cases, \madam~biases towards larger $v_t$ than \adam~but does not require $v_t$ to be monotonically increasing, which is like an adaptive version of \amsgrad. To highlight the difference, we compare \adam, \madam~and \amsgrad~on the synthetic dataset from~\cite{chen2018convergence} simulating training with stochastic mini batches on a finite set of samples.
Formally, let $\mathbbm{1}_{[\cdot]}$ be the indicator function. We consider the problem $\min_{\theta} f(\theta)=\sum_{i=1}^{11} \ell_{i}(\theta)$ where
\begin{equation}\label{eq:toy_converge}
\ell_i(\theta)=\left\{ 
    \begin{array}{ll}
      \mathbbm{1}_{i=1}5.5\theta^2+\mathbbm{1}_{i\neq 1}(-0.5\theta^2), & \text{if } |\theta|\le 1; \\
     \mathbbm{1}_{i=1}(11|\theta|-5.5)+\mathbbm{1}_{i\neq 1}(-|\theta|+0.5), & \text{otherwise}. 
    \end{array}
    \right.
\end{equation}
At every step, a random index $i$ is sampled uniformly from $i\in [11]$, and the gradient $\nabla \ell_i(\theta)$ is used by the optimizer.
The only stationary point where $\nabla f(\theta)=0$ is $\theta=0$. 
We set $\alpha=0, \beta=0.9$ for \adam~and \amsgrad. For \madam, we set $\alpha=0, (\underline{\beta}, \bar{\beta})=(0.5,1)$. We select the best \textit{constant} learning rates for the three algorithms, see Appendix~\ref{sec:a_finite_samples} for details.

We plot the median and standard error of the objective ($f(\theta)$), accumulated update size ($S_1=\sum_{t=1}^{T}|| g_t/\sqrt{v_t}||^2$), and total change in adaptive step size ($S_2=\sum_{t=1}^{T}|| \frac{1}{\sqrt{v_t}}-\frac{1}{\sqrt{v_{t-1}}} ||_1$) over 100 runs in Figure~\ref{eq:toy_converge}.The optimal learning rates for these optimziers are different, so for fair comparisons, we have ignored the constant learning rate in $S_1$ and $S_2$. 
From the curves of $f(\theta)$, we can see \adam~diverges, and \madam~converges faster than \amsgrad~in the later stage. 
As shown by the $S_2$ curves, the adaptive step sizes of \madam~and \amsgrad~all converged to some constant values after about 10 steps, but \madam~converges faster on both $f(\theta)$ and $S_1$, indicating the adaptive step size found by \madam~fits the geometry of the problem better than \amsgrad. 
This also shows $S_1+S_2$ of \madam~has a smaller slope than \amsgrad~in the log-scale plots after 10 iterations, leading to a faster theoretical convergence rate in the bound given by~\cite{chen2018convergence}.
The slightly larger variation in adaptive step sizes of \madam~at the beginning of training, shown by the larger $S_2$ values, demonstrates \madam~adapts faster to the changing gradients than \amsgrad, achieved by dynamically selecting $\beta<0.9$.

\subsection{Convergence in the Noisy Quadratic Model}
\begin{figure*}[h]
\vspace{-1em}
\centering
\includegraphics[width=.4\linewidth]{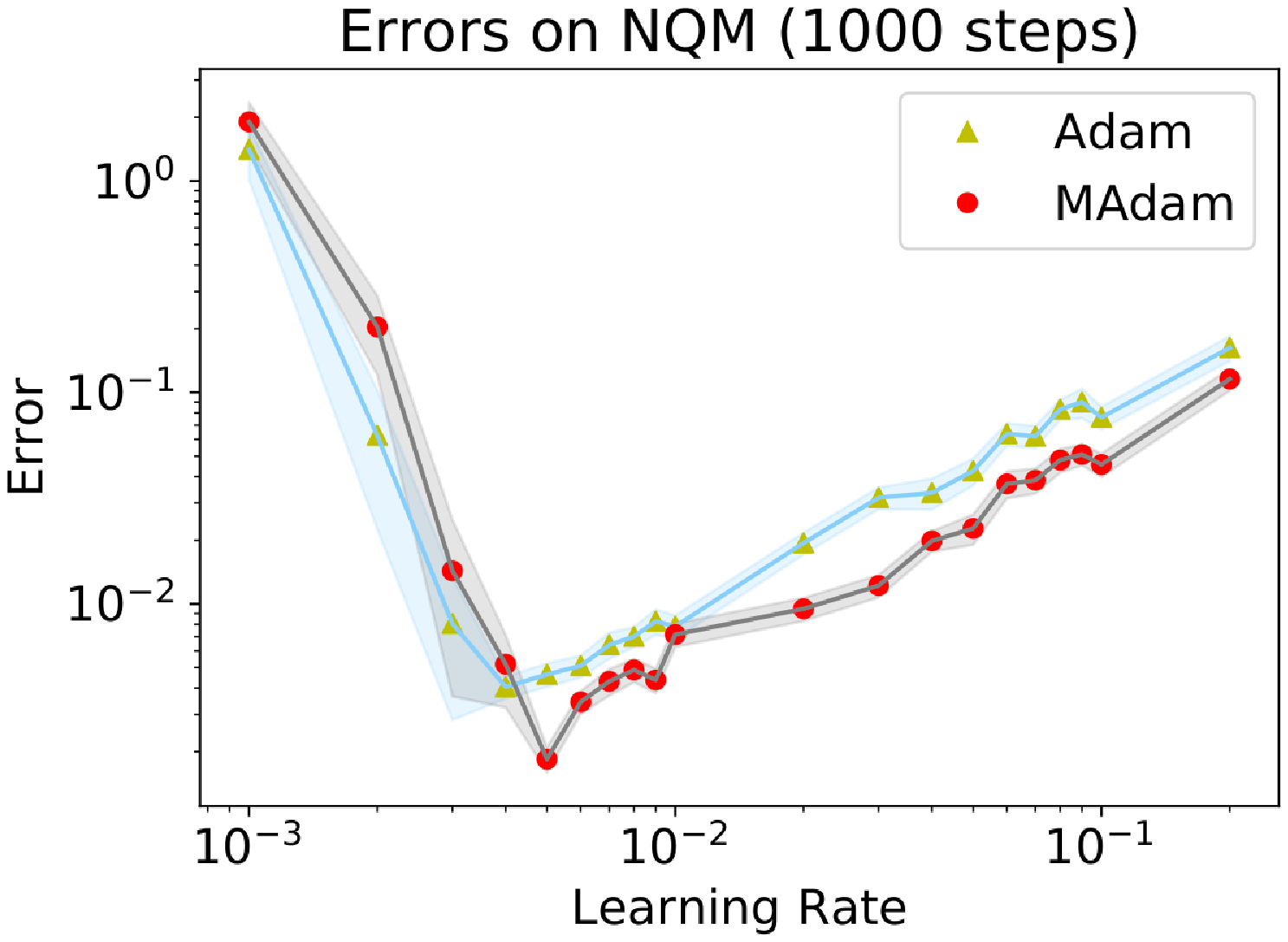}
\includegraphics[width=.4\linewidth]{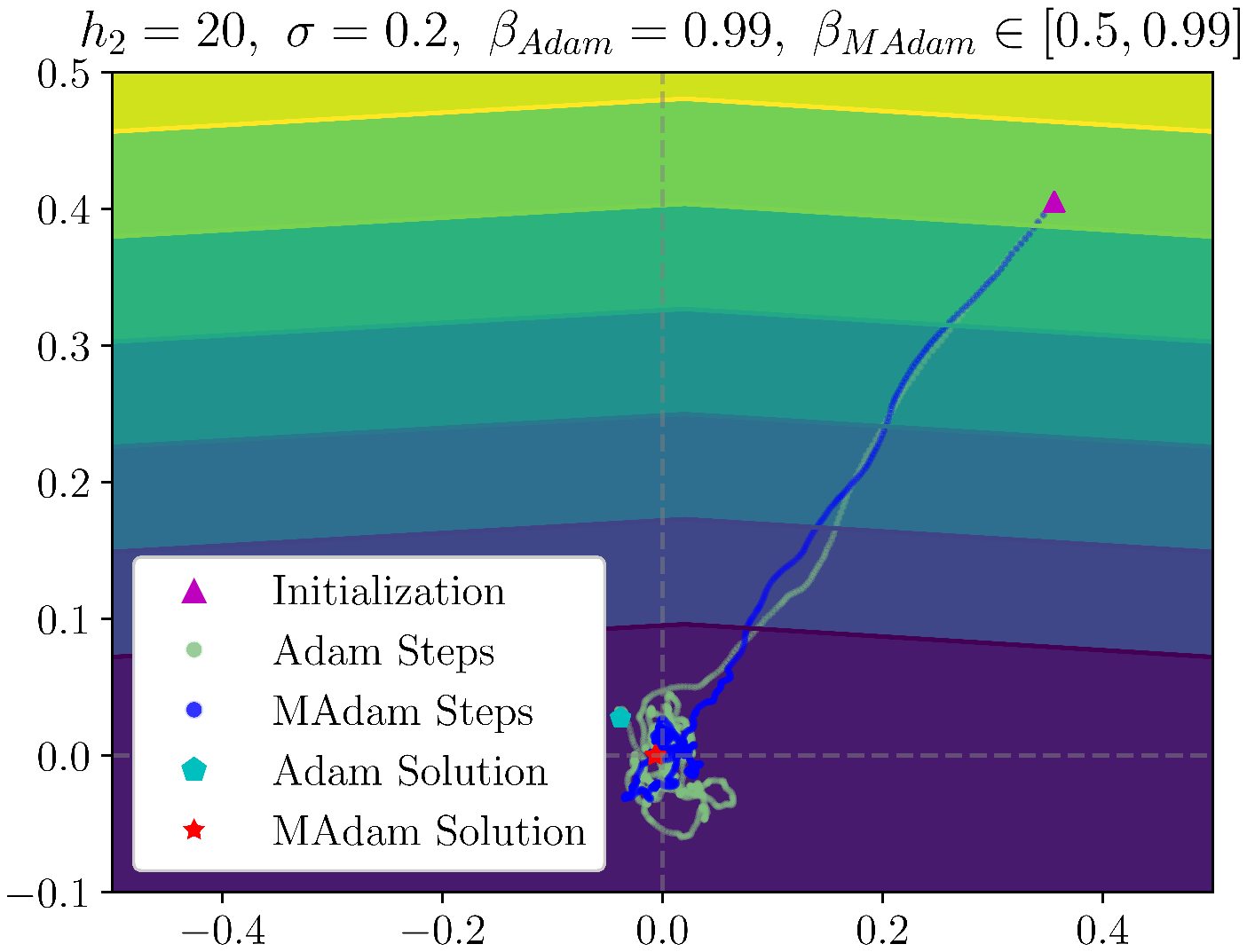}
\vspace{-1em}
\caption{\small Results on NQM. The left figure shows the mean and standard error of the loss under different learning rates $\eta$, computed over 100 runs at each point. We select the best $\beta$ for \adam~at each $\eta$. The best results (mean and variance) of \adam~and \madam~are 1.84e-3 (2.51e-4) and 4.05e-3 (4.84e-4) respectively. Figure on the right gives a qualitative example of the trajectories of two approaches. }
\label{fig:nqm}
\end{figure*}

We analyze the ability of \madam~to adapt to curvature and gradient noise on the simple but illustrative Noisy Quadratic Model (NQM), which has been widely adopted for analyzing optimization dynamics~\cite{schaul2013pesky,wu2018shorthorizon,zhang2019algorithmic,lookahead}.
The loss function is defined as $f(\theta) = \mathbb{E}_{x\sim \gN(0,\sigma^2 I)}\left[\frac{1}{2}\sum_{i=1}^d h_i(\theta_i - x_i)^2 \right]$,
where $x$ is a noisy observation of the ground-truth parameter $\theta^*=0$, simulating the gradient noise in stochastic optimization, and $h_i$ represents the curvature of the system in $d$ dimensions.
In each step, the optimizers use the following noisy gradient for coordinate $i$, from which we can see the gradient's variance is proportional to the curvature $h_i^2$:
\begin{equation}\label{eq:nqm_grad}
    \nabla_{\theta_i} \ell(\sigma\epsilon_i;\theta_i)= h_i(\theta_i - \sigma\epsilon_i), \epsilon_i \sim \gN(0, 1).
\end{equation}

To validate the effectiveness of MaxVA, we compare \madam~with \adam~under a variety of different curvatures $h$ and noise level $\sigma$ on an NQM with $d=2$.
For each setting of $h$ and $\sigma$, we test both algorithms on a variety of learning rates. For $\adam$, we additionally choose the best $\beta$ and report the best results. See Appendix~\ref{appendix:nqm} for details. We run each setting 100 times to report the mean and standard error.
\madam~consistently achieves 30-40\% lower average loss with smaller standard error in all settings. 
Figure~\ref{fig:nqm} shows the results for one of the settings, from which we find the best result of \madam~is better than \adam~under any choice of $\beta$ and learning rate, confirming the advantage of MaxVA. 
From the qualitative example, MaxVA also demonstrates smaller variance near convergence, enabled by a quicker response to impede the noise with a smaller $\beta_t$.
More experimental results under other settings are provided in Appendix~\ref{appendix:nqm}.

\section{Experiments on Practical Datasets}

In this section, we evaluate \madam~and \lamadam~on a variety of tasks against well-calibrated baselines: IWSLT'14 DE-EN/WMT'16 EN-DE for neural machine translation, the GLUE benchmark for natural language understanding, and pretraining the BERT-Base model. 
We also provide results on image classification.
We use the decoupled weight decay~\cite{loshchilov2018adamw} in all our experiments. 
Across all the plots in this section, we define the average step size at time $t$ as the average of $|\eta_t m_t/(\sqrt{v_t}+\epsilon)|$ for \adam/\madam~and $|\eta_t m_t|$ for \laprop/\lamadam~over all the entries.

\subsection{Image Classification}
\begin{table}[t!]
\small
\begin{center}
\label{tab:layer}
\begin{tabular}{llll}
\toprule
Model     & CIFAR-10     & CIFAR-100   & ImageNet \\ 
\midrule
 SGD      & 95.44 (.04)  & 79.62 (.07) & 70.18 \\
 \midrule
 \adam    & 95.37 (.03)    & 78.77 (.07) & 66.54 \\ 
 \laprop  & 95.34 (.03)  & 78.36 (.07) & 70.02  \\ 
 AdaBelief & 95.30$^*$    & 77.30$^*$    &  70.08  \\
 \madam{} (ours)   & \textbf{95.51} (.09)  & \textbf{79.32} (.08) & 69.96  \\ 
 \lamadam{} (ours) & 95.38 (.11)  & 79.21 (.11) & \textbf{70.16}\\
 \bottomrule
 \end{tabular}
 \end{center}
 \vspace{-1em}
 \caption{\small Comparing adaptive methods with exhaustively fine-tuned SGD on CIFAR10/100 and ImageNet.  CIFAR10/100 experiments are the median (standard error) over 4 runs. $^*$: The results of AdaBelief are from their paper~\cite{zhuang2020adabelief} with a ResNet34, while our results are with ResNet18.}
 \label{tab:image}
  \vspace{-2em}
\end{table}
To evaluate the effectiveness of MaxVA for image classification, we compare with SGD, \adam{}, \laprop{}~\cite{ziyin2020laprop} and AdaBelief~\cite{zhuang2020adabelief} in training ResNet18~\cite{he2016resnet} on CIFAR10, CIFAR100 and ImageNet.
On all the datasets, we perform a grid search for the learning rate and weight decay, and report the best results for each method in Table~\ref{tab:image}. 
For CIFAR10/100, we train ResNet18 with a batch size of 128 for 200 epochs. 
We also find AMSGrad~\cite{reddi2019amsgrad} improves the classification accuracy of all adaptive methods evaluated on CIFAR10/100, so we apply AMSGrad in all experiments with adaptive methods.
On ImageNet, we use the implementation from torchvision and the default multi-step learning rate schedule. We do not use AMSGrad in this case. Further details are in Appendix~\ref{appendix:image}. 

Despite achieving a marginal improvement on CIFAR10, adaptive methods often underperforms carefully tuned SGD on CIFAR100 and ImageNet when training popular architectures such as ResNet, as confirmed by~\cite{wilson2017marginal,lookahead,liu2019radam}.
Nevertheless, with the proposed MaxVA, we shrink the gap between adaptive methods and carefully tuned SGD on these image classification datasets, and achieve top-1 accuracy very close to SGD on ImageNet. 
Note our results with ResNet18 is better than the recent AdaBelief's results with ResNet34 on CIFAR10/CIFAR100 (95.51/79.32 vs. 95.30/77.30 approximately), as well as AdaBelief with ResNet18 on ImageNet (70.16 vs. 70.08)~\cite{zhuang2020adabelief}.

\subsection{Neural Machine Translation}

\begin{table}[t!]
\small
\centering
\begin{tabular}{lcc}
\toprule
Method & IWSLT'14 DE-EN & WMT'16 EN-DE\\
\midrule
RAdam    & 35.51   &   -          \\
AdaBelief & 35.90   &   -          \\
\midrule
\laprop (ours) & 35.98 (0.06) & 27.02 \\
\lamadam (ours) & \textbf{36.09} (0.04) & \textbf{27.11} \\
\bottomrule 
\end{tabular}
\caption{\small BLEU score for training transformers on machine translation datasets. We report the median and standard error for IWSLT'14 over 5 runs. Results of other meethods are from the AdaBelief paper~\cite{zhuang2020adabelief}.} 
\label{tab:nmt}
\vspace{-2em}
\end{table}
We train Transformers from scratch with \laprop~and \lamadam~on IWSLT'14 German-to-English (DE-EN) translation~\cite{cettolo2014iwslt} and WMT'16 English-to-German (EN-DE) translation, based on the implementation of fairseq.\footnote{https://github.com/pytorch/fairseq}
We do not compare with SGD, since it is unstable for Transformers~\cite{zhang2019longtail}.
We also show in Appendix~\ref{sec:adabound} that \adabound~cannot achieve any good result without degenerating into \adam.
More details are in Appendix~\ref{appendix:nmt}.

IWSLT'14 DE-EN has 160k training examples, on which we use a Transformer with 512-dimensional word embeddings and 1024 FFN dimensions.
We train it for 60k iterations, with up to 4096 tokens in each minibatch. Results are listed in Table~\ref{tab:nmt}. Note the baseline's BLEU score is already 1.22 higher than the best results reported in~\cite{liu2019radam} using the same model. 
As shown in Appendix~\ref{appendix:nmt}, \lamadam~uses much smaller update size than \laprop, and it is not able for \laprop~to achieve better results even when we scale its learning rate to get similar update sizes as \lamadam, indicating MaxVA helps to find a better minimum not achievable by using constant $\beta$.

WMT'16 EN-DE has 4.5M training examples, where same as~\cite{ott2018scaling}, we use a larger Transformer with 1024-dimensional word embeddings and 4096 FFN dimensions. 
Each batch has up to 480k tokens. We train for 32k iterations using the same inverse square root learning rate schedule as~\cite{transformer}.
We evaluate the \textit{single model} BLEU on newstest2013, unlike~\cite{liu2019radam} where models in the last 20 epochs are averaged to get the results. 
As shown in Table~\ref{tab:nmt}, \lamadam~also achieves better results.

\subsection{General Language Understanding Evaluation (GLUE)}
\vspace{-1em}
\begin{table*}[htbp!]
\setlength\tabcolsep{2pt} 
\small
\centering
  \begin{adjustbox}{max width=\textwidth}
\begin{tabular}{lcccccccc}
\toprule
{Method}   & \multicolumn{1}{c}{{MNLI}} & \multicolumn{1}{c}{{QNLI}} & \multicolumn{1}{c}{{QQP}} & \multicolumn{1}{c}{{RTE}} & \multicolumn{1}{c}{{SST-2}} & \multicolumn{1}{c}{{MRPC}} & \multicolumn{1}{c}{{CoLA}} & \multicolumn{1}{c}{{STS-B}} \\
                  & \multicolumn{1}{c}{(Acc)}         & \multicolumn{1}{c}{(Acc)}         & \multicolumn{1}{c}{(Acc)}        & \multicolumn{1}{c}{(Acc)}        & \multicolumn{1}{c}{(Acc)}          & \multicolumn{1}{c}{(Acc)}         & \multicolumn{1}{c}{(Mcc)}         & \multicolumn{1}{c}{(Pearson)}      \\ 
\midrule
{Reported}   &           87.6        &          92.8	         &    91.9  &      78.7                &            94.8          &        90.2         &           63.6         &  91.2                  \\ 
\midrule
{\adam}   &            87.70 (.03)    &     92.85 (.06)  &    91.80 (.03)  &     79.25 (.71)               &      94.75 (.08)         &       88.50 (.24)      &             61.92 (1.1)           &       91.17 (.13)       \\ 
{\laprop}      &     87.80 (.04)  &      92.85 (.13)      &     91.80 (.03)       &             78.00 (.46)            &   94.65 (.11)     &     89.20 (.20)                   &       63.01 (.61)        &        91.17 (.06)       \\ 

\madam   &          \textbf{87.90} (.08)     &     92.95 (.07)      &    91.85 (.03)                &      79.60 (.66)        &    94.85 (.12)    &                   89.70 (.17)      &          63.33 (.60)       &     91.28 (.03)    \\ 

\lamadam   & {87.80} (.03)                         & \textbf{93.05} (.05)                         & \textbf{91.85} (.05)                       & \textbf{80.15} (.64)                       & \textbf{95.15} (.15)                          & \textbf{90.20} (.20)                      & \textbf{63.84} (.85)                         & \textbf{91.36} (.04)                        \\ 
\bottomrule
\end{tabular}
\end{adjustbox}
\caption{\small Results (median and variance) on the dev sets of GLUE based on finetuning the RoBERTa-base model (\cite{liu2019roberta}), from 4 runs with the same hyperparameter but different random seeds.}
\label{tab:glue}
\vspace{-2em}
\end{table*}
\begin{figure}[t!]
\centering
\includegraphics[width=.3\linewidth]{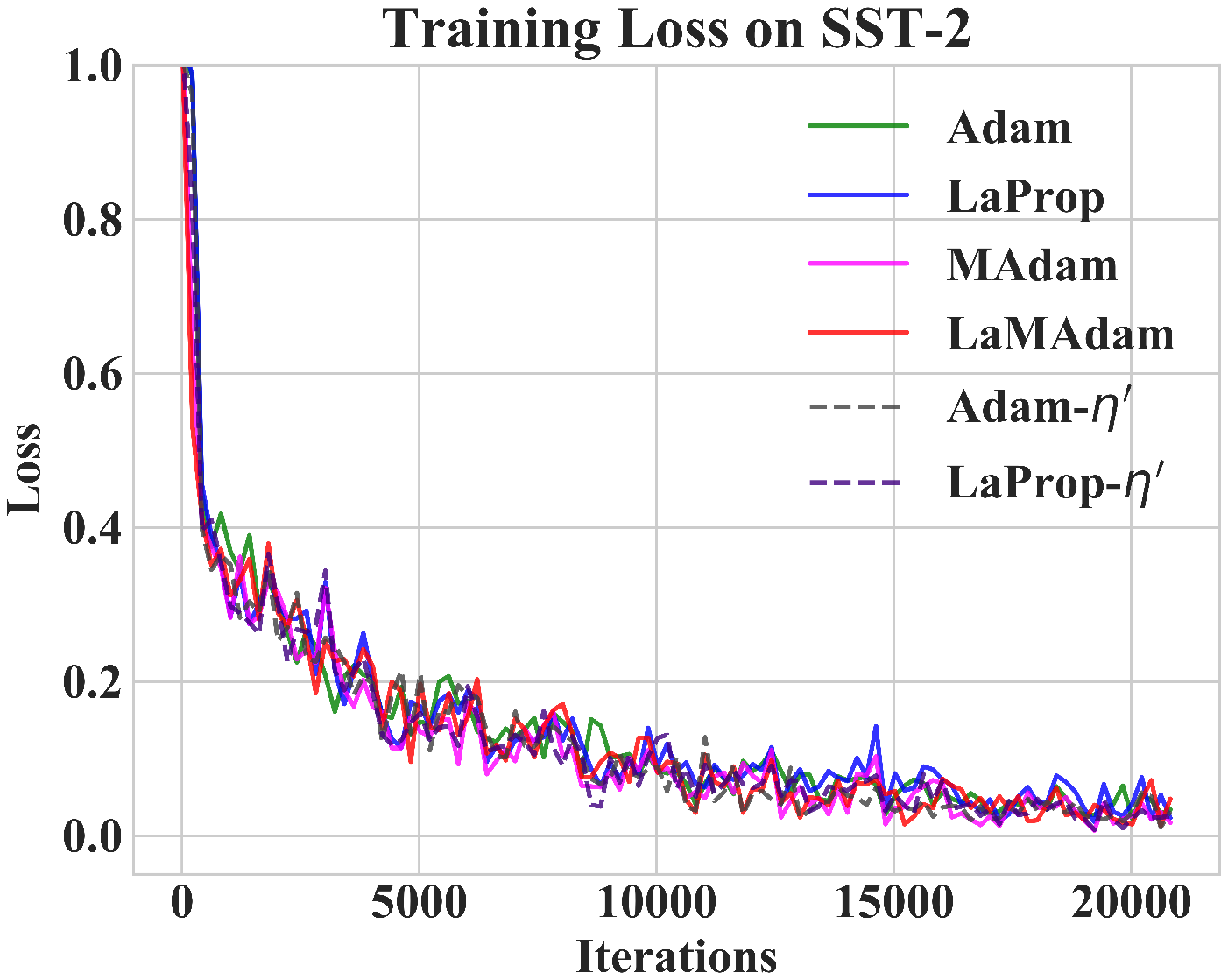}
\includegraphics[width=.3\linewidth]{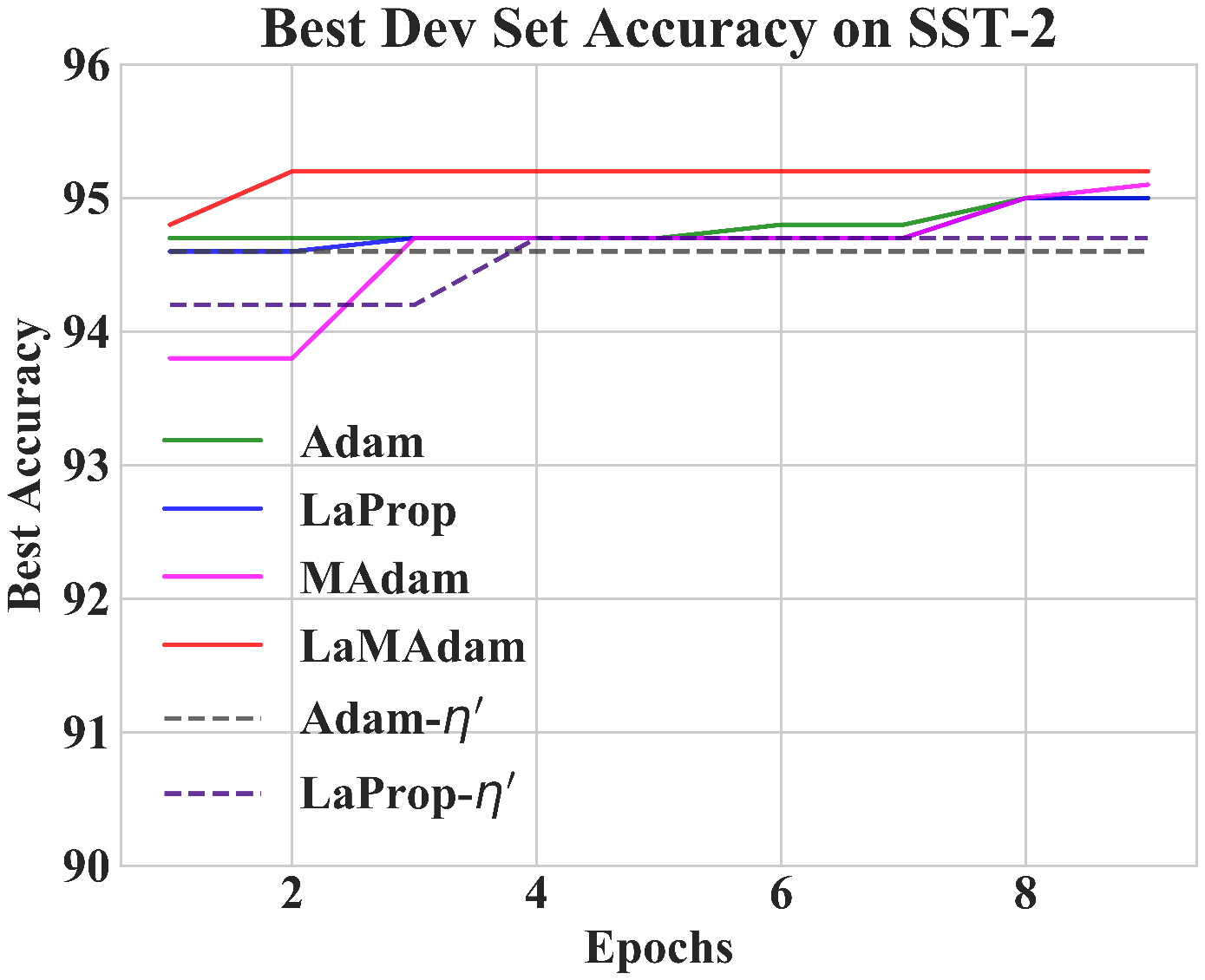}
\includegraphics[width=.3\linewidth]{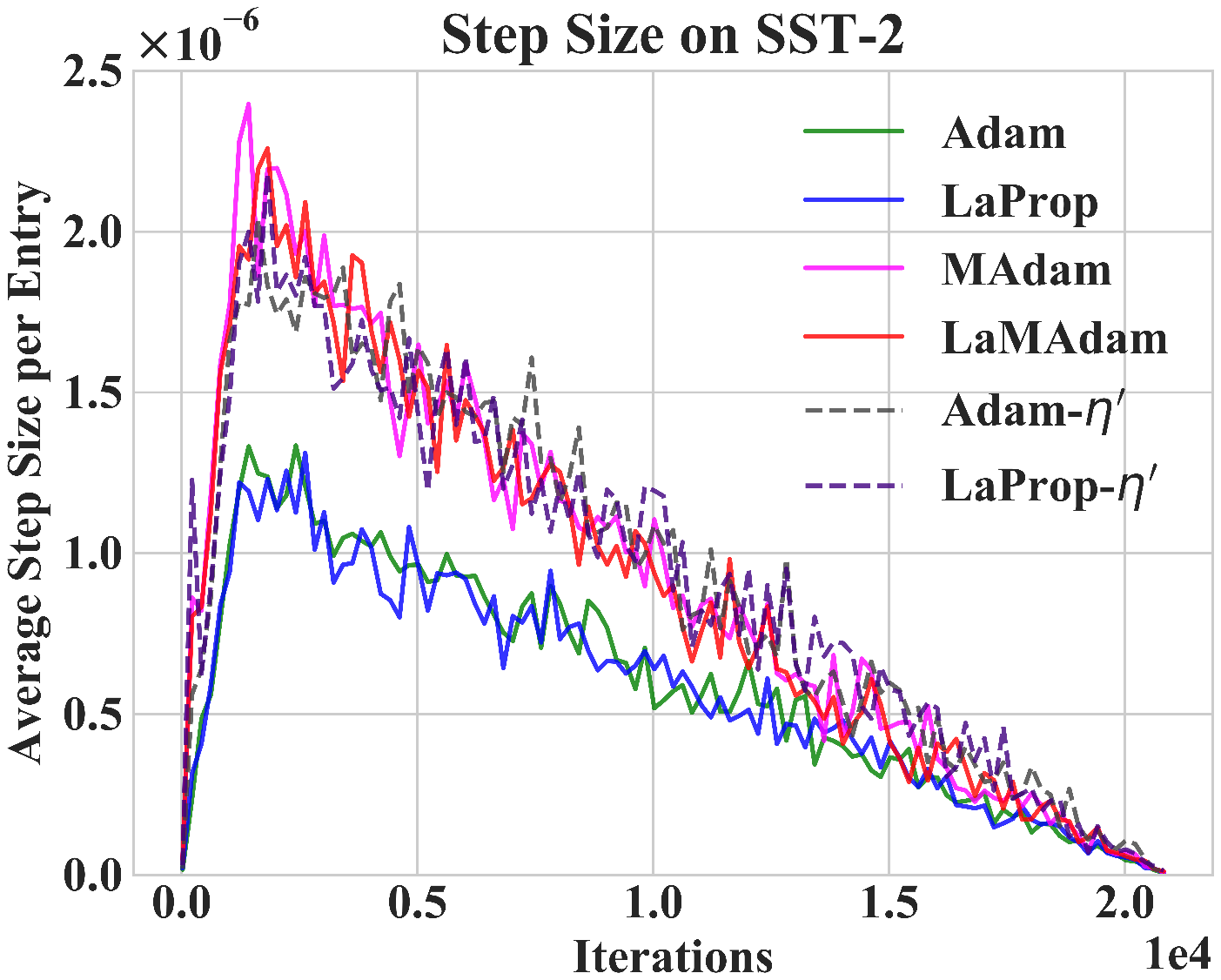}
\caption{\small Training loss, validation accuracy and step size of various optimization methods on SST-2. All optimizers here use $\lambda=0.1$. \adam~and \laprop~use ($\eta, \beta$)=(1e-5, 0.98), \madam~and \lamadam~use ($\eta,\underline{\beta},\bar{\beta}$)=(4e-5, 0.5, 0.98), \adam-$\eta'$ and \laprop-$\eta'$ use ($\eta,\beta$)=(1.6e-5, 0.98).}
\label{fig:nlu_convergence}
\vspace{-1em}
\end{figure}
To evaluate MaxVA for transfer learning, we fine-tune pre-trained RoBERTa-base model~\cite{liu2019roberta} on 8 of the 9 tasks of the GLEU benchmark~\cite{wang2018glue}.
Following prevalent validation settings~\cite{devlin2019bert,lan2019albert,raffel2019t5}, we report the median and standard error for fine-tuning the RoBERTa-base model~\cite{liu2019roberta} over 4 runs where only the random seeds are changed. The results are in Table~\ref{tab:glue}. 
\madam~and \lamadam~give better scores than the corresponding baselines in the 8 tasks.
More experimental details are in Appendix~\ref{appendix:glue}.

To highlight the difference of the optimizers, we compare the training loss, dev set accuracy and the average step size on SST-2, as shown in Figure~\ref{fig:nlu_convergence}.
Different from Machine Translation experiments where we train the Transformers from scratch, the adaptive step size of \madam/\lamadam~is higher in this transfer learning setting. The ratio of the learning rate and step size of MaxVA to non-MaxVA optimizers are 4 and 1.8 respectively on GLUE, while on IWSLT'14 the two ratios are 2 and (approximately) 0.875. 
Because we start from a pre-trained model, the heavy tail of the gradient is alleviated, just as the BERT model in the later stage of training as shown by~\cite{zhang2019longtail}, and the curvature of the loss landscape should be smaller.
Therefore, MaxVA selects larger adaptive step sizes for better convergence.
Same as in the Machine Translation experiments, the highest test accuracy of \adam/\laprop~cannot reach the same value as \madam/\lamadam~by simply scaling the base learning rate $\eta$ to reach similar step sizes as \madam/\lamadam. 

\subsection{Large-batch Pretraining for BERT}
We use the NVIDIA BERT pretraining repository to perform large-batch pretraining for BERT-Base model on the Wikipedia Corpus only.\footnote{Note the results from the repository are for BERT-Large trained with additional data from BookCorpus.}
Each run takes about 52 hours on 8 V100 GPUs.
Training is divided into two phases: the first phase uses a batch size of 64K with input sequence length 128 for 7,038 steps; the second phase uses a batch size 32K with input sequence length 512 for 1563 steps. The total of steps is significantly smaller than the 1,000,000 steps used in the small-batch training of~\cite{devlin2019bert}. 
Therefore, a faster adaptation to curvature in each step is more important. 

\begin{figure*}
  \begin{center}
    \includegraphics[width=0.35\linewidth]{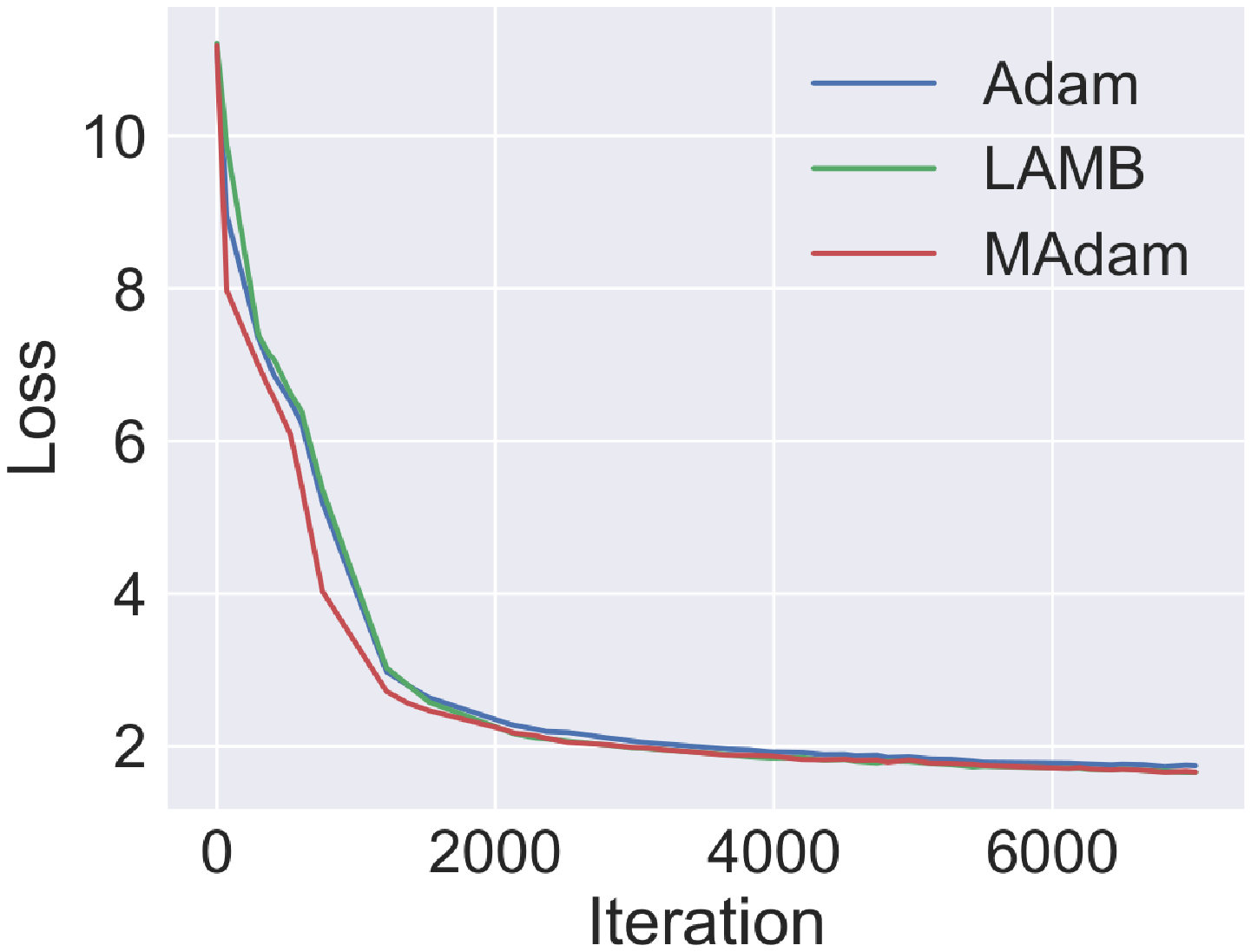}
    \includegraphics[width=0.35\linewidth]{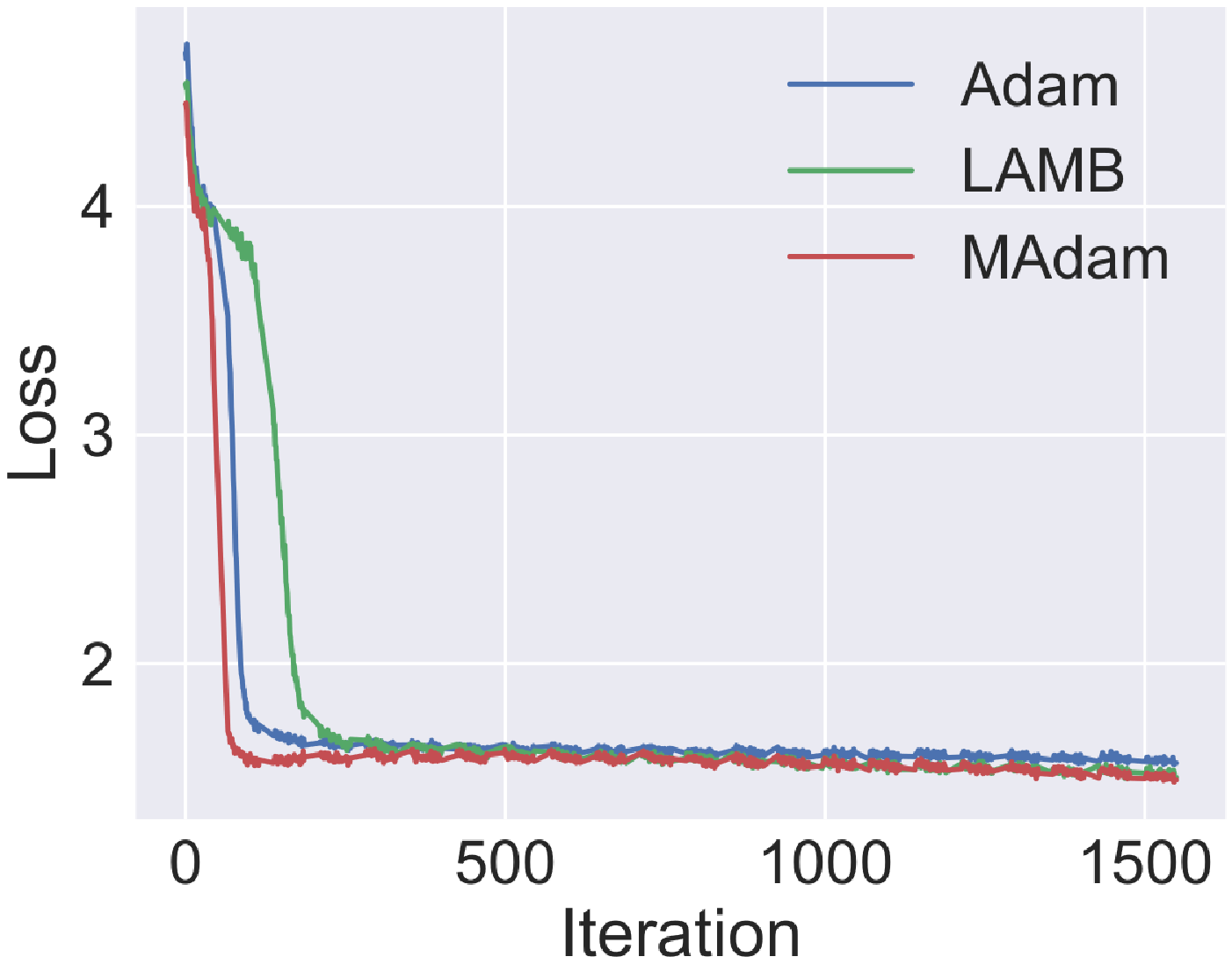}
  \end{center}
  \vspace{-1em}
  \caption{\footnotesize Training losses of \adam, \lamb~and \madam~on Wikipedia Corpus in the two training phases.}
  \vspace{-1em}
\label{fig:lars_convergence}
\end{figure*}

This point is validated by the faster convergence of \madam~in both phases, as shown in the training loss curves in Figure~\ref{fig:lars_convergence}. 
Contrary to the observation by~\cite{you2020lamb}, \adam~even converges faster than \lamb~in the earlier iterations. 
\cite{you2020lamb} only explored weight decay of up to 0.01 for \adam, but we find using larger weight decay of 0.1 together with gradient clipping ($\lVert g_t \rVert_2\le 1$, same as \lamb) stabilizes \adam. 
We inherit this setting for \madam.
For \madam~and \adam, we do a grid search on the learning rate of phase 1 while keeping the ratios of learning rate in phase 1 and phase 2 to the same as \lamb.
We use $\bar{\beta}=0.999, \underline{\beta}=0.5$ for \madam.
For \lamb, we use the default setting from the aforementioned repository.

The faster adaptation of MaxVA improves the stability, which enables \madam{} to use a much larger learning rate to achieve faster convergence than \adam.
The best learning rate for \madam~is 3.4e-3. We tried learning rates in \{7e-4, 8e-4, 9e-4, 1e-3\} for \adam, and find it always diverges when the learning rate is higher or equal to 9e-4. The best result of \adam~is achieved with learning rate 8e-4.
\madam~achieves a training loss of 1.492, while \lamb~achieves a training loss of 1.507, and \adam~has the worst training loss 1.568. 
The test scores of the models pretrained with \madam/\lamb/\adam~are 88.53/87.60/88.07 (F1) and 82.10/81.40/80.78 (Accuracy) on SQuAD v1.1 and MNLI, respectively.

\section{Related Work}\label{sec:related}

Various adaptive methods have been proposed and broadly applied in deep learning \cite{adam,adagrad,Tieleman2012rmsprop,adadelta}. 
\cite{reddi2019amsgrad} proposed to compute the adaptive learning rate with the coordinate-wise maximum value of $v_t$ so that the adaptive learning rate does not increase. 
\adabound~\cite{luo2019adabound}~clips the adaptive learning rate of \adam~with a decreasing upper bound and an increasing lower bound. Lookahead \cite{lookahead} computes weight updates by looking ahead at the sequence of ``fast weights" generated by another optimizer. 
Padam \cite{Padam} improves the generalization of adaptive methods by choosing a proper exponent for the $v_t$ of \amsgrad. LAPROP \cite{ziyin2020laprop} uses local running estimation of the variance to normalize the gradients, resulting in higher empirical stability.
RAdam~\cite{liu2019radam} was recently invented to free \adam~from the warmup schedule for training Transformers. \cite{ma2019adequacy} found that using a linear warmup over $2\cdot(1-\beta_2)^{-1}$ iterations for \adam~achieves almost the same convergence as RAdam.
\cite{lars} proposes Layer-wise Adaptive Rate Scaling (LARS), and scales the batch size to 16,384 for training ResNet50. 
LAMB~\cite{you2020lamb} applies a similar layer-wise learning rate on \adam~ to improve LARS on training BERT.
Starting from a similar motivation of adapting to the curvature, the recent work AdaBelief~\cite{zhuang2020adabelief} directly estimates the exponential running average of the gradient deviation to compute the adaptive step sizes. 
Our approach finds the averaging coefficients $\beta_t$ automatically by maximizing the estimated variance for a faster adaptation to the curvature, which could be complementary to all the aforementioned methods, and is the first to explore in this direction to our knowledge.


\vspace{-3mm}
\section{Conclusion}
\vspace{-3mm}
In this paper, we present Maximum Variation Averaging (MaxVA), a novel adaptive learning rate scheme that replaces the exponential running average of squared gradient with an adaptive weighted mean. 
In each step, MaxVA chooses the weight $\beta_t$ for each coordinate, such that the esimated gradient variance is maximized. 
This enables MaxVA to: (1) take smaller steps when large curvatures or abnormally large gradients are present, which leads to more desirable convergence behaviors in face of noisy gradients; (2) adapt faster to the geometry of the objective, achieving faster convergence in the large-batch setting.
We illustrate how our method improves convergence by a better adaptation to variance, and demonstrate strong empirical results on a wide range of tasks. 
We prove MaxVA converges in the nonconvex stochastic optimization setting under mild assumptions.


%
%
%
\bibliographystyle{splncs04}
\bibliography{refs}

\clearpage
\appendix
\onecolumn
\section*{MaxVA: Fast Adaptation of Step Sizes by Maximizing Observed Variance of Gradients (Appendix)}
\section{Deriving the closed form solution Eq.\ref{eq:beta_sol}}\label{appendix:closed_form}
Plugging Eq.~\ref{eq:zeroth},\ref{eq:first},\ref{eq:second}, and the unbiased estimations $u_t(\beta)=\tilde{u}_{t}(\beta) / w_t(\beta), v_t(\beta)=\tilde{v}_{t}(\beta) / w_t(\beta)$
into Eq.~\ref{eq:max-form}, each coordinate is solving the same problem:
\begin{equation}\label{eq:alt_obj}
\argmax_{\beta} f(\beta) = \frac{\beta w_{t-1} {v}_{t-1} + (1-\beta) g_t^2}{\beta w_{t-1} + (1-\beta)} - \left[\frac{\beta w_{t-1} {u}_{t-1}+(1-\beta)g_t}{\beta w_{t-1} + (1-\beta) } \right]^2.
\end{equation}
Let $\gamma = {1}/[\beta w_{t-1}+(1-\beta)]\in [1, 1/w_{t-1}]$, we can see $f(\beta)$ can be represented as a quadratic function of $\gamma$. Specifically,
\begin{equation*}
    f(\beta) = h(\gamma) = \frac{w_{t-1}v_{t-1}-g_t^2}{w_{t-1}-1}+\left[g_t^2-\frac{w_{t-1}v_{t-1}-g_t^2}{w_{t-1}-1}\right] \gamma -\left\{ \frac{w_{t-1}u_{t-1}-g_t}{w_{t-1}-1} + \left[ g_t - \frac{w_{t-1}u_{t-1}-g_t}{w_{t-1}-1} \right]\gamma \right\}^2.
\end{equation*}
Meanwhile, $\beta$ is a monotonic function of $\gamma$. 
Therefore, $f(\beta)$ has a unique maximum value. 

To find the maximum value, we return to Eq.~\ref{eq:alt_obj}, from which we can find a stationary point
\begin{equation}\label{eq:correct}
    \frac{v_{t-1} - u_{t-1}^2 + (g_t-u_{t-1})^2}{ w_{t-1}\left[ (g_t-u_{t-1})^2 - v_{t-1} + u_{t-1}^2 \right]+v_{t-1} - u_{t-1}^2 + (g_t - u_{t-1})^2  }.
\end{equation}

\section{Convergence Proof}
\label{sec:convergence_proof}
Following the convergence proofs of \yogi~\cite{zaheer2018yogi}, we prove the convergence of \madam~in the nonconvex setting. 
\paragraph{Proof of Theorem 1.}
\begin{proof}
Recall that we have assumed the update steps of \madam~as
\begin{equation}
    \theta_{t+1, i} = \theta_{t,i} - \eta_t \frac{g_{t,i}}{\sqrt{v_{t,i}}+\epsilon},
\end{equation}
for all $i\in [d]$, and that $f$ is $L$-smooth, which results in the following inequalities:
\begin{equation}\label{eq:lsmooth}
\begin{split}
    f(\theta_{t+1}) &\le f(\theta_t) + \langle \nabla f(\theta_t), \theta_{t+1} - \theta_t \rangle + \frac{L}{2}\lVert \theta_{t+1} - \theta_t \rVert^2 \\
    & = f(\theta_t) - \eta_t \sum_{i=1}^d \nabla f(\theta_{t,i}) \frac{g_{t,i}}{\sqrt{v_{t,i}}+\epsilon} +\frac{L\eta_t^2}{2} \sum_{i=1}^d \frac{g_{t,i}^2}{(\sqrt{v_{t,i}} +\epsilon)^2}.
\end{split}
\end{equation}
Note the stochastistic gradient is defined as $g_t=\frac{1}{b_t}\sum_{j=1}^{b_t} \nabla_\theta \ell(x_j;\theta_t)$.
Given $\theta_t$, we take expectation over the stochastic gradient $g_t$ in Eq.~\ref{eq:lsmooth} (denoted as $\E_t[\cdot]=\E[\cdot|\theta_t]$) to get
\begin{equation}\label{eq:lsmooth2}
\begin{split}
    \E_t[f(\theta_{t+1})]\le & f(\theta_t) - \eta_t \sum_{i=1}^d \left( [\nabla f(\theta_t)]_i \times \E_t\left[ \frac{g_{t,i}}{\sqrt{v_{t,i}}+\epsilon} \right] \right) + \frac{L\eta_t^2}{2}\sum_{i=1}^d \E_t\left[ \frac{g_{t,i}^2}{(\sqrt{v_{t,i}}+\epsilon)^2} \right] \\
    = & f(\theta_t) - \eta_t \sum_{i=1}^d \left( [\nabla f(\theta_t)]_i \times \E_t\left[ \frac{g_{t,i}}{\sqrt{v_{t,i}}+\epsilon} -\frac{g_{t,i}}{\sqrt{\beta_{t,i} v_{t-1,i}}+\epsilon} + \frac{g_{t,i}}{\sqrt{\beta_{t,i} v_{t-1,i}}+\epsilon} \right] \right) \\
    & + \frac{L\eta_t^2}{2}\sum_{i=1}^d \E_t\left[ \frac{g_{t,i}^2}{(\sqrt{v_{t,i}}+\epsilon)^2} \right] \\
     = & f(\theta_t) - \eta_t \sum_{i=1}^d \left( [\nabla f(\theta_t)]_i \times \left[\E_t\left[\frac{g_{t,i}}{\sqrt{\beta_{t,i} v_{t-1,i}}+\epsilon} \right] + \E_t\left[ \frac{g_{t,i}}{\sqrt{v_{t,i}}+\epsilon} -\frac{g_{t,i}}{\sqrt{\beta_{t,i} v_{t-1,i}}+\epsilon}  \right]\right] \right) \\
    & + \frac{L\eta_t^2}{2}\sum_{i=1}^d \E_t\left[ \frac{g_{t,i}^2}{(\sqrt{v_{t,i}}+\epsilon)^2} \right] \\
    \le & f(\theta_t) -\eta_t \sum_{i=1}^d \left[\frac{[\nabla f(\theta_t)]_i^2}{\sqrt{\bar{\beta} v_{t-1,i}}+\epsilon} - \frac{\sigma G}{\epsilon\sqrt{t}}\right]  + \eta_t \sum_{i=1}^d [\nabla f(\theta_t)]_i  \E_t \underbrace{ \left[ \left| \frac{g_{t,i}}{\sqrt{v_{t,i}}+\epsilon} -\frac{g_{t,i}}{\sqrt{\beta_{t,i} v_{t-1,i}}+\epsilon}\right|  \right]}_{T_1} \\
    & + \frac{L\eta_t^2}{2}\sum_{i=1}^d \E_t\left[ \frac{g_{t,i}^2}{(\sqrt{v_{t,i}}+\epsilon)^2} \right],  \\
\end{split}
\end{equation}
where the second equality holds by applying Lemma~\ref{lemma} to the first expectation term, and taking the absolute value of the second expectation term.

Next, we need to bound the term $T_1$ to show convergence. 
First, we have the following upper bound for $T_1$:
\begin{equation}
\begin{split}
    T_1 = & \left|{\frac{g_{t,i}}{\sqrt{v_{t,i}}+\epsilon}} - {\frac{g_{t,i}}{\sqrt{\beta_{t,i} v_{t-1,i}}+\epsilon}}\right| \\
    \le & {|g_{t,i}|} \times {\left| { \frac{1}{\sqrt{v_{t,i}}+\epsilon} } -  { \frac{1}{\sqrt{\beta_{t,i}v_{t-1,i}}+\epsilon} }\right| } \\
     \le & { \frac{|g_{t,i}|}{ (\sqrt{v_{t,i}}+\epsilon)(\sqrt{\beta_{t,i}v_{t-1,i}}+\epsilon)  }} \times {\left| \frac{v_{t,i}-\beta_{t,i} v_{t-1,i}}{\sqrt{v_{t,i}}+\sqrt{\beta_{t,i}v_{t-1,i}}} \right|} \\
     = & { \frac{|g_{t,i}|}{ (\sqrt{v_{t,i}}+\epsilon)(\sqrt{\beta_{t,i}v_{t-1,i}}+\epsilon)  }} \times { \frac{(1 - \beta_{t,i}) g_{t,i}^2}{\sqrt{v_{t,i}}+\sqrt{\beta_{t,i}v_{t-1,i}}}}, \\
\end{split}
\end{equation}
where the last equality comes from the definition of $v_{t,i}=\beta_{t,i} v_{t-1,i} + (1-\beta_{t,i}) g_{t,i}^2$.
We can further bound $T_1$ as 
\begin{equation}
\begin{split}
    T_1 = & \frac{g_{t,i}^2}{(\sqrt{v_{t,i}}+\epsilon)(\sqrt{\beta_{t,i}v_{t-1,i}}+\epsilon)} \times { \frac{(1-\beta_{t,i})|g_{t,i}|}{\sqrt{\beta_{t,i}v_{t-1,i}+(1-\beta_{t,i})g_{t,i}^2 }+\sqrt{\beta_{t,i}v_{t-1,i}}} } \\
    \le & \frac{g_{t,i}^2}{(\sqrt{\beta_{t,i}v_{t-1,i}}+\epsilon)\epsilon}\times { \frac{(1-\beta_{t,i})|g_{t,i}|}{\sqrt{(1-\beta_{t,i})g_{t,i}^2 }} } \\
    = & \frac{\sqrt{1-\beta_{t,i}} g_{t,i}^2}{(\sqrt{\beta_{t,i}v_{t-1,i}}+\epsilon)\epsilon} \le \frac{\sqrt{1-\underline{\beta}} g_{t,i}^2}{(\sqrt{\underline{\beta} v_{t-1,i}}+\epsilon)\epsilon}
\end{split}
\end{equation}
Since the loss on each sample $s$ satisfies $|[\nabla \ell (x,s)]_i|\le G$, we will have $|[\nabla f(x)]_i|\le G$ for $\forall i\in [d]$. 
Substituting the coefficients of $T_1$ in Eq.~\ref{eq:lsmooth2} with this gradient bound, we have
\begin{equation}
\begin{split}
    \E_t[f(\theta_{t+1})]\le & f(\theta_t) - {\eta_t \sum_{i=1}^d \frac{[\nabla f(\theta_t)]_i^2}{\sqrt{\bar{\beta}  v_{t-1,i}}+\epsilon}}  + \frac{\eta_t G\sqrt{1-{\underline{\beta}}}}{\epsilon} \sum_{i=1}^d \E_t\left[ \frac{ g_{t,i}^2}{\sqrt{{\underline{\beta}} v_{t-1,i}}+\epsilon} \right] \\
    & + \frac{L\eta_t^2}{2}\sum_{i=1}^d \E_t\left[ \frac{g_{t,i}^2}{(\sqrt{v_{t,i}}+\epsilon)^2} \right]  + \frac{\sigma\eta dG}{\epsilon\sqrt{t}},  \\ 
    \le & f(\theta_t) - {\eta_t \sum_{i=1}^d \frac{[\nabla f(\theta_t)]_i^2}{\sqrt{\bar{\beta} v_{t-1,i}}+\epsilon}}  + \frac{\eta_t G\sqrt{1-{\underline{\beta}}}}{\epsilon} \sum_{i=1}^d \E_t \left[\frac{ g_{t,i}^2}{\sqrt{\underline{\beta} v_{t-1,i}}+\epsilon}\right]  \\
    & + \frac{L\eta_t^2}{2\epsilon}\sum_{i=1}^d \E_t\left[ \frac{g_{t,i}^2}{\sqrt{\underline{\beta}v_{t-1,i}}+\epsilon} \right]  + \frac{\sigma\eta dG}{\epsilon\sqrt{t}},  \\ 
    \le & f(\theta_t) + \sum_{i=1}^d \underbrace{\left( - \frac{\eta_t}{(\sqrt{\bar{\beta} v_{t-1,i}}+\epsilon)} + \frac{\eta_t G\sqrt{1-\underline{\beta}}}{\epsilon(\sqrt{\underline{\beta} v_{t-1,i}}+\epsilon)} + \frac{L\eta_t^2}{2\epsilon (\sqrt{\underline{\beta} v_{t-1,i}}+\epsilon)} \right)}_{T_2} [\nabla f(\theta_t)]_i^2  \\
    & + \left( \frac{\eta_t G\sqrt{1-\underline{\beta}}}{\epsilon} + \frac{L\eta_t^2}{2\epsilon} \right) \sum_{i=1}^d \frac{\sigma^2}{b_t(\sqrt{\underline{\beta} v_{t-1,i}}+\epsilon)}  + \frac{\sigma\eta dG}{\epsilon\sqrt{t}},   \\ 
\end{split}
\end{equation}
where the second inequality comes from the fact that $\sqrt{v_{t,i}} + \epsilon \ge  \epsilon$ and $v_{t,i}=\beta_{t,i}v_{t-1,i} + (1-\beta_{t,i})g_{t,i}^2\ge \underline{\beta}v_{t-1,i}$, and the third inequality comes from applying Lemma 1 by \cite{zaheer2018yogi} to $\E_t[g_t^2]$. 
The application of Lemma 1 is possible because $v_{t-1}$ is independent of the $t$-th batch.
By the assumptions for $\epsilon, G, \underline{\beta}$, we have  
\begin{equation}
    \frac{L\eta_t^2}{2\epsilon} \le \frac{\eta}{4},~ \frac{\eta G\sqrt{1-\underline{\beta}}}{\epsilon} \le \frac{1}{4}\eta.
\end{equation}
Plugging these two results and the assumption $\bar{\beta}\le 2\underline{\beta}$ into $T_2$, we have
\begin{equation}
\begin{split}
    T_2 & \le  -\frac{\eta}{\sqrt{\bar{\beta} v_{t-1,i}} + \epsilon} + \frac{\eta}{2( \sqrt{\underline{\beta}v_{t-1,i}} + \epsilon )} \\
        & \le -\frac{\eta}{\sqrt{2}(\sqrt{\underline{\beta} v_{t-1,i}} + \epsilon)} + \frac{\eta}{2( \sqrt{\underline{\beta}v_{t-1,i}} + \epsilon )} \\
        & \le \frac{\eta}{5( \sqrt{\underline{\beta} v_{t-1,i}} + \epsilon )}
\end{split}
\end{equation}

the main inequality, we have 
\begin{equation}
    \begin{split}
         \E_t[f(\theta_{t+1})]\le & f(\theta_t) - \frac{\eta}{5}\sum_{i=1}^d\frac{[\nabla f(\theta_t)]_i^2}{\sqrt{\underline{\beta} v_{t-1,i}}+\epsilon} + \left( \frac{\eta_t G\sqrt{1-\underline{\beta}}}{\epsilon} + \frac{L\eta_t^2}{2\epsilon} \right)  \sum_{i=1}^d \frac{\sigma^2}{t(\sqrt{\underline{\beta}v_{t-1,i}}+\epsilon)} + \frac{\sigma\eta dG}{\epsilon\sqrt{t}} \\
         \le & f(\theta_t) - \frac{\eta}{5 (G\sqrt{\underline{\beta}} + \epsilon)}\lVert \nabla f(\theta_t) \rVert^2 + \left( \frac{\eta_t G\sqrt{1-\underline{\beta}}}{\epsilon^2} + \frac{L\eta_t^2}{2\epsilon^2} \right)  \frac{\sigma^2 d}{t}+ \frac{\sigma\eta dG}{\epsilon\sqrt{t}},
    \end{split}
\end{equation}
where we have replaced $b_t$ with $t$ by our assumption on the batch size, and the second inequality comes from the fact that $v_{t-1,i}\le G^2$.
Taking expectation on both the LHS and RHS for the inequalities at $t=1,...,T$, using telescope sum and rearranging the terms, we can conclude that
\begin{equation}
    \frac{\eta}{5(G\sqrt{\underline{\beta}}+\epsilon)}\sum_{i=1}^T \lVert \nabla f(\theta_t) \rVert^2 \le f(\theta_1) - E[f(\theta_{T+1})] + \left( \frac{\eta G\sqrt{1-\underline{\beta}}}{\epsilon^2} + \frac{L\eta^2}{2\epsilon^2} \right) \sigma^2 d\log(T+1) + \frac{2\sigma\eta dG}{\epsilon}\sqrt{T}.
\end{equation}
Multiplying both sides with $\frac{5(G\sqrt{\underline{\beta}}+\epsilon)}{T\eta}$, and using the fact that $f(x^*)\le f(\theta_{t+1})$, we conclude that
\begin{equation}
    \frac{1}{T}\sum_{i=1}^T \lVert \nabla f(\theta_t) \rVert^2 \le  5(G\sqrt{\underline{\beta}} + \epsilon)\left( \frac{f(\theta_1) - f(x^*)}{\eta T} + \left( \frac{ G\sqrt{1-\underline{\beta}}}{\epsilon^2} + \frac{L\eta}{2\epsilon^2} \right)  \frac{\sigma^2 d\log(T+1)}{T} + \frac{2\sigma dG}{\epsilon \sqrt{T}} \right).
\end{equation}
\end{proof}

\newcommand{\gfti}{[\nabla f(\theta_t)]_i}
\newcommand{\vtmi}{v_{t-1,i}}
\newcommand{\bti}{\beta_{t,i}}
\newcommand{\gti}{g_{t,i}}

\begin{lemma}\label{lemma}
Assume the gradient is bounded as $\lVert \nabla_\theta \ell(x;\theta) \rVert_\infty \le G$, and has bounded variance $\E[[\nabla_\theta \ell(x;\theta)]_i - \gfti]^2\le \sigma^2$, and the batch size $b_t=t$. For the $t$-th iteration of \madam, we have 
\begin{equation}
    -[\nabla f(\theta_t)]_i \E_t \left[ \frac{g_{t,i}}{\sqrt{\beta_{t,i} v_{t-1, i}} + \epsilon} \right] \le -\frac{[\nabla f(\theta_t)]_i^2 }{\sqrt{\bar{\beta}v_{t-1, i}} + \epsilon} + \frac{G \sigma}{\epsilon\sqrt{t}}
\end{equation}
\end{lemma}
\begin{proof}
The LHS can be decomposed as
\begin{equation}\label{eq:lemma_decomp}
    \begin{split}
        \text{LHS} & = -\gfti \E_t\left[ \frac{\gfti }{\sqrt{\bti\vtmi} + \epsilon}\right] - \gfti \E_t\left[\frac{\gti - \gfti}{\sqrt{\bti \vtmi} + \epsilon}\right]\\
                   & \le - \frac{\gfti^2 }{\sqrt{\bar{\beta}\vtmi} + \epsilon} \underbrace{ - \gfti \E_t\left[\frac{\gti - \gfti}{\sqrt{\bti \vtmi} + \epsilon}\right]}_{T_3},
    \end{split}
\end{equation}
where the inequality comes from taking the upper bound of $\beta_{t,i}$, since the first term is non-positive.
Let $[h(x)]_{+}$ and $[h(x)]_{-}$ be the operators for taking the positive and negative values of function $h(x)$ respectively, i.e.,
\begin{equation}
    [h(x)]_{+} = 
    \begin{cases}
      h(x), & \text{if}\ h(x)>0 \\
      0, & \text{otherwise}
    \end{cases}, ~~
    [h(x)]_{-} = 
    \begin{cases}
      h(x), & \text{if}\ h(x)<0 \\
      0, & \text{otherwise}
    \end{cases}.
\end{equation}
It is obvious that $\E[[X]_+]\le \E[|X|]\le \sqrt{\E[X^2]}$, where the second inequality comes from Cauchy-Schwarz inequality.
Similarly, $\E[[X]_-]\ge -\E[|X|]\ge -\sqrt{\E[X^2]}$.
With this in mind, we have 
\begin{equation}
    0 \le \E_t\left[\left[\gti - \gfti \right]_+\right] \le \sqrt{\E_t\left[ \gti - \gfti \right]^2} \le \frac{\sigma}{\sqrt{t}}, 
\end{equation}
where the last inequality comes from applying Lemma 1 from~\cite{zaheer2018yogi} under the bounded gradient variance assumption, and the assumption that the batch size grows as $b_t=t$. 
Similarly, we have 
\begin{equation}
    -\frac{\sigma}{\sqrt{t}} \le \E_t\left[\left[\gti - \gfti \right]_-\right] \le 0.
\end{equation}
Now we will decompose and bound $T_3$ as
\begin{equation}
\begin{split}
    T_3 & = - \E_t\left[ \gfti \frac{[\gti - \gfti]_+}{\sqrt{\bti \vtmi} + \epsilon}\right] - \E_t\left[\gfti \frac{ [\gti - \gfti]_-}{\sqrt{\bti \vtmi} + \epsilon}\right] \\
        & \le \begin{cases}
              - \E_t\left[\gfti \frac{ [\gti - \gfti]_-}{\sqrt{\underline{\beta} \vtmi} + \epsilon}\right], & \text{if}\ \gfti>0 \\
              - \E_t\left[\gfti \frac{ [\gti - \gfti]_+}{\sqrt{\underline{\beta} \vtmi} + \epsilon}\right], & \text{otherwise}
            \end{cases} \\
        & \le  \frac{\sigma |\gfti|}{(\sqrt{\underline{\beta}\vtmi} + \epsilon)\sqrt{t}}, \\
        & \le \frac{\sigma G}{\epsilon\sqrt{t}}.
\end{split}
\end{equation}
Plugging this inequality back into Eq.~\ref{eq:lemma_decomp} and we will get the RHS.

\end{proof}


\begin{algorithm}[H]\small
	\caption{\lamadam}
	\label{alg:LaMadam}
	\begin{algorithmic}[1]
		\State {\bfseries Input:} Learning rate $\{\eta_t\}_{t=1}^T$, parameter $0 < \alpha < 1$, $0<\underline{\beta}<\bar{\beta}<1$, $\epsilon > 0$
		\State Set $\tilde{m}_{0}=\tilde{u}_0=\tilde{v}_{0}=w_0 = 0$
		\For{$t=1$ {\bfseries to} $T$}
		\State Draw samples $S_t$ from training set
        \State Compute $g_t = \frac{1}{|S_t|} \sum_{x_k \in \mathcal{S}_t}\nabla \ell(x_k; \theta_t)$
        \State $\tilde{\beta}_t={\arg\max}_\beta v_t(\beta)-u_t^2(\beta)$ \Comment{see Eq \ref{eq:beta_sol}}
        \State $\beta_t = \max(\underline{\beta}, \min(\bar{\beta}, \tilde{\beta}_t))$
        \State $\tilde{u}_t=\beta_t \tilde{u}_{t-1}+(1-\beta_t)g_t$
        \State $\tilde{v}_t=\beta_t \tilde{v}_{t-1}+(1-\beta_t)g^2_t$
        \State $w_t=\beta_t w_{t-1}+(1-\beta_t)$
        \State $\tilde{m}_{t} = \alpha \tilde{m}_{t-1} + (1 - \alpha) \frac{ g_{t}}{\sqrt{\tilde{v}_t/w_t}+\epsilon}$
		\State $\theta_{t} = \theta_{t-1} - \frac{\eta_t}{1-\alpha^t} \tilde{m}_t $
		
		\EndFor
	\end{algorithmic}
\end{algorithm}

\section{Practical notes of $\beta_t$}\label{sec:div}
\paragraph{Claims and arguments:}
\begin{enumerate}
\item \textit{For $t>1,~\text{since }0<\beta_t\le 1$, $w_t$ will monotonically increase from $(1-\beta_1)$ to 1.}

This is obvious since in every step, $w_{t}$ is an interpolation between $w_{t-1}$ and $1$, and $w_t \ge w_{t-1}$. We have also set $w_1=1-\beta_1$.

\item \textit{For any $g_t, u_{t-1}, v_{t-1}$ satisfying $v_{t-1}-u_{t-1}^2 > 0$ in Eq.~\ref{eq:beta_sol}, we have $ \beta_t \in [1/(1+w_{t-1}), 1/(1-w_{t-1})]$.}

Eq.~\ref{eq:beta_t} is monotonic in $R_t$ .Since $g_t$ can be any value, $R_t$ can be any value from 0 to $\infty$. If $R_t=0$, $\beta_t$ takes the largest value $1/(1-w_t)$. If $R_t \rightarrow \infty$, $\beta_t\rightarrow 1/(w_t+1)$.

\item \textit{As $t \rightarrow \infty$, $w_t\rightarrow 1$ and $\beta_t\in [0.5, \infty]$.}\\
Combining Claims 1 and 2 to get this result.

\item \textit{Adding a small coefficient $\delta>0$ to the denominator of Eq.~\ref{eq:beta_sol} has negligibale effect on the value of $\beta_t$ and does not violate the maximum variation objective (Eq.~\ref{eq:max-form}).}\\
Since $\delta$ is small, it has negligible effect on $\beta_t$ when division by zero does not happen. 
We only need to confirm adding $\delta$ will not affect the solution when division by zero happens.  
We can re-write the dividend of Eq.~\ref{eq:beta_sol} as
\begin{equation}\label{eq:zero}
    (w_{t-1}+1)(g_t-u_{t-1})^2+(1 - w_{t-1})(v_{t-1}-{u}_{t-1}^2).
\end{equation}
Since $\mathbb{E}[X^2]-\left(\mathbb{E}[X]\right)^2=\text{Var}[X]\ge 0$, we can conclude that ${v}_{t-1}-{u}^2_{t-1}\ge 0$.

When $1-\beta_1 \le w_{t-1}<1$, Eq.~\ref{eq:zero} can be 0 only when $g_t={u}_{t-1}$ and ${v}_{t-1}={u}^2_{t-1}$. 
In this special case, we can set $\beta_t$ to any value in $[0,1]$ without changing ${\sigma}_t^2$; we will always have ${v}_t=\tilde{v}_{t-1}/w_{t-1}=v_{t-1}, u_t=\tilde{u}_{t-1}/w_{t-1}=u_{t-1}$, and ${\sigma}_t^2=0$. 
Only $w_t=(w_{t-1}-1)\beta_t + 1$ is affected by $\beta_t$, which takes a larger value when $\beta_t$ is smaller.
The solution given by adding $\delta$ to the denominator is $\beta_t=0$, and the following clipping will set $\beta_t=\underline{\beta}$, resulting in the largest possible $w_t=(w_{t-1}-1)\underline{\beta}+1$. 
In the next step, if Eq.~\ref{eq:zero} is not zero, then we have $\beta_{t+1}=1/(w_t+1)$, and we know $g_{t+1}\neq u_{t}$.\footnote{Otherwise we will still have $g_{t+1}=u_{t}, g_{t+1}^2=u_{t}^2=v_t$ and Eq.~\ref{eq:zero} is 0.}
In this case, for $0.5\le \beta_{t+1}<1$, ${\sigma}_{t+1}^2$ increases as $\beta_{t+1}$ decreases, so setting $w_t$ to its maximum will achieve the maximum variance at the next step.
Otherwise if Eq.~\ref{eq:zero} is zero, doing this will not change ${\sigma}_{t+1}^2=0$.

When $w_{t-1}=1$, Eq.~\ref{eq:zero} is 0 if and only if $g_t={u}_{t-1}$. 
As a result, if $v_{t-1}=u_{t-1}^2$, we have the same conclusion as before. 
Otherwise, $\beta_t=({{v}_{t-1}-{u}_{t-1}^2})/{\delta}$ before clipping, and $\beta_t=\bar{\beta}$ after clipping.
Also, any $0<\beta_{t}<1$ will not change the value of ${u}_{t}=\beta_t u_{t-1}+(1-\beta_t) g_t=u_{t-1}$. 
Since $g_t^2=u_{t-1}^2<v_{t-1}$, to maximize ${\sigma}_t^2=v_t(\beta)-u_{t-1}^2$, we should set $\beta_t = \bar{\beta}$ so that $v_t(\beta)$ takes the maximum value, which is consistent with the solution after adding $\delta$ to the denominator. 
\end{enumerate}

\section{\adabound~might fail on Transformers?}
\label{sec:adabound}
\begin{figure}
\vspace{-2em}
        \begin{center}
            \includegraphics[width=0.3\textwidth]{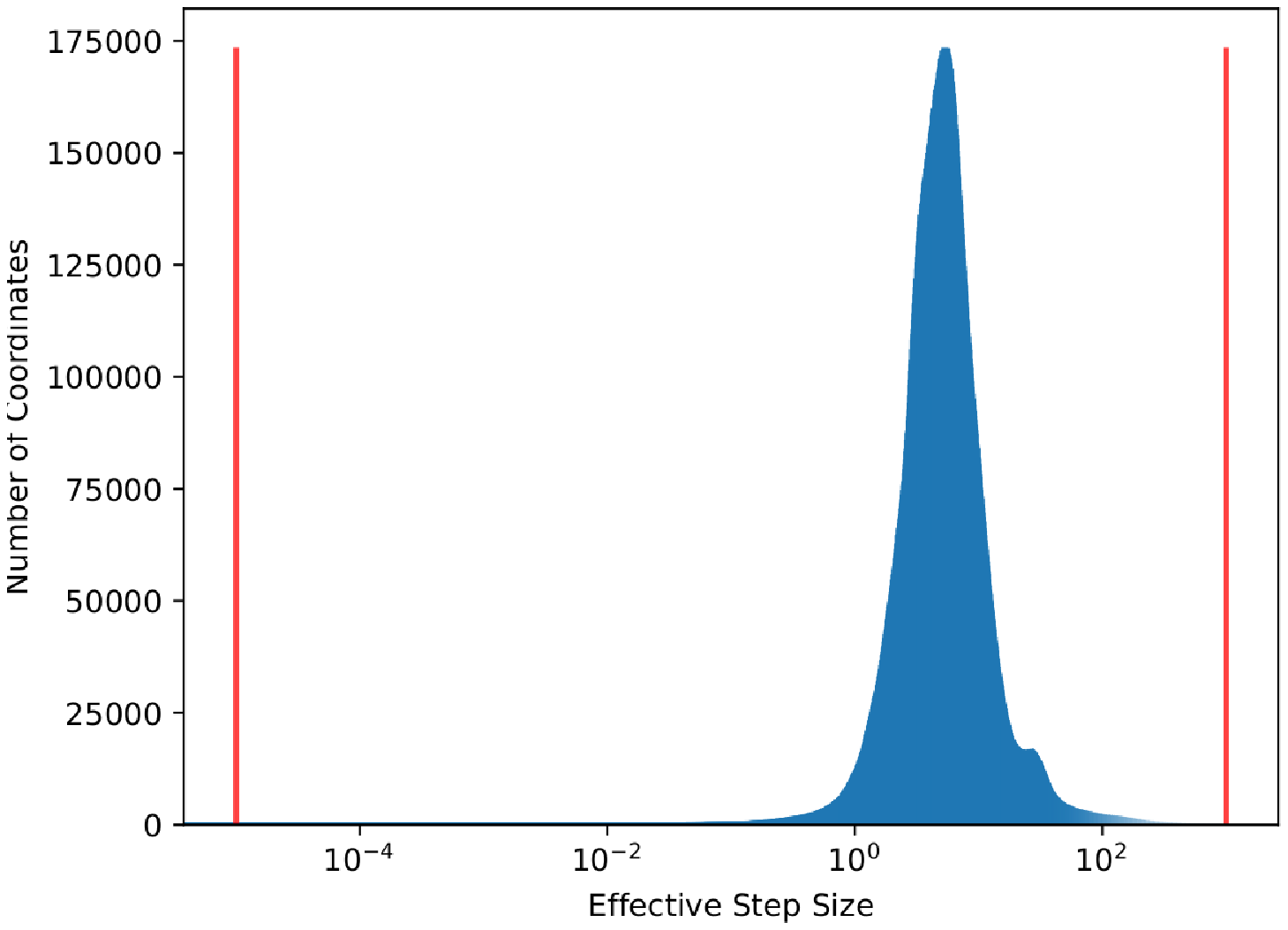}
            \includegraphics[width=0.3\textwidth]{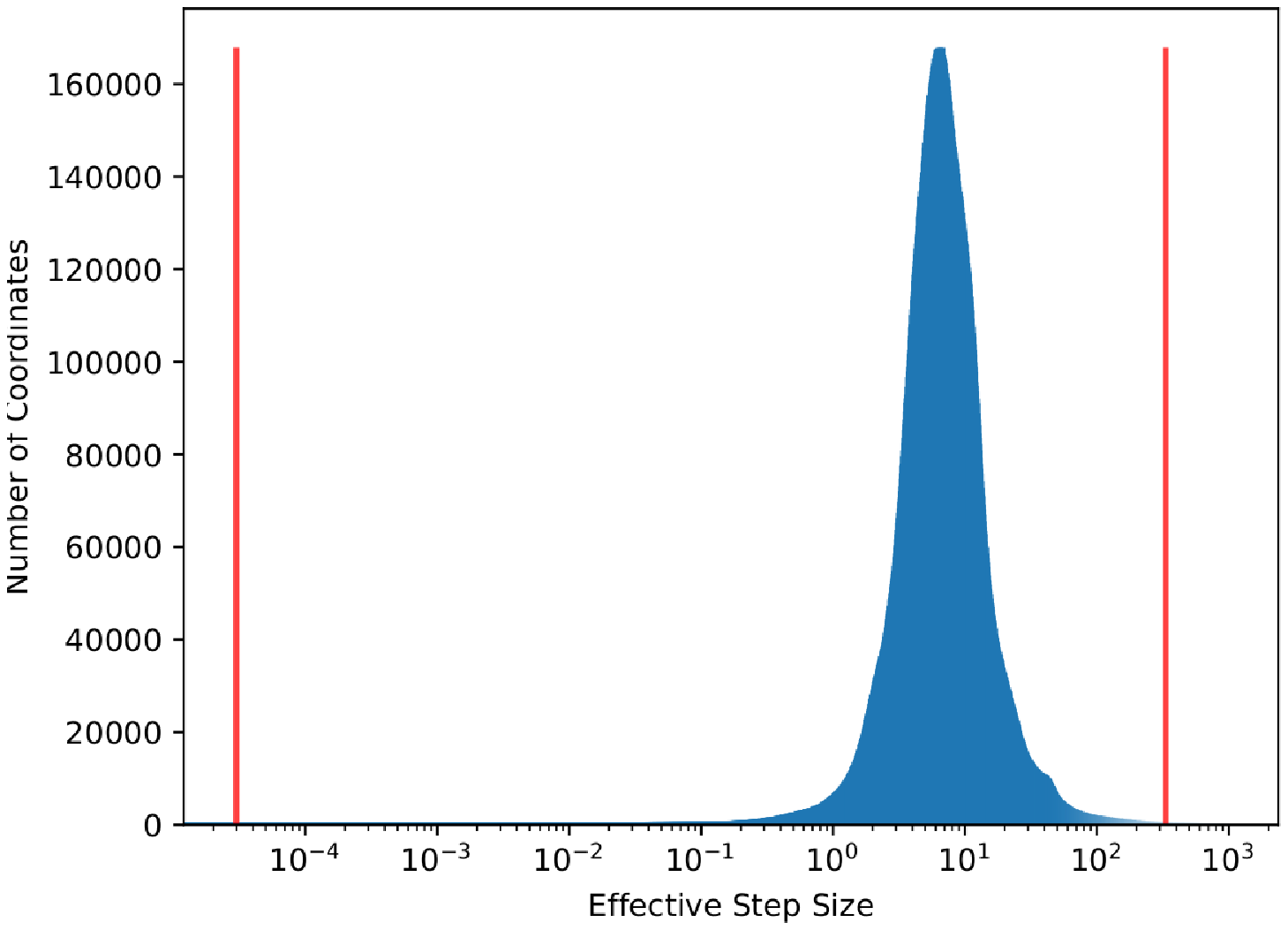}
            \includegraphics[width=0.3\textwidth]{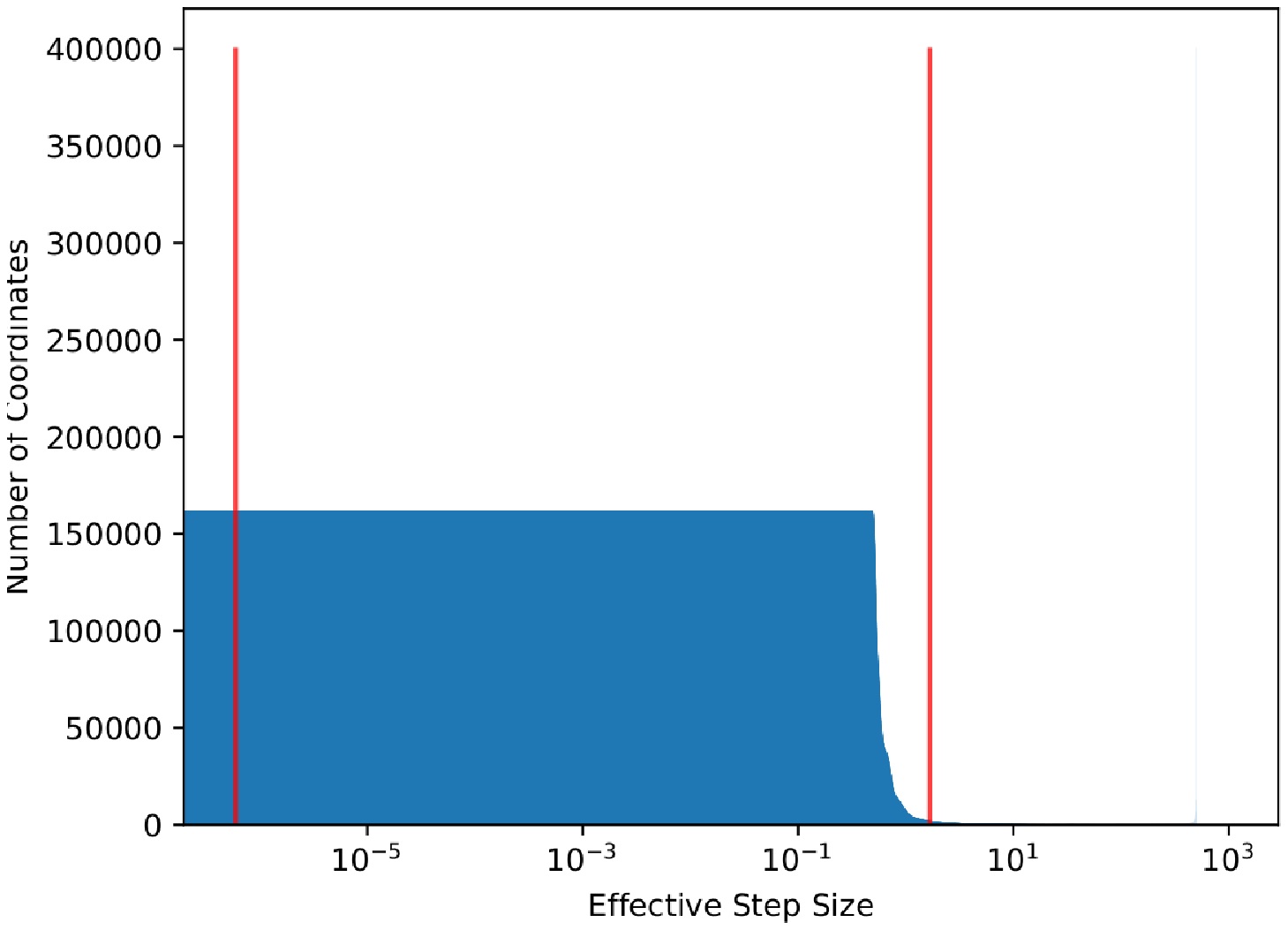}
            \includegraphics[width=0.3\textwidth]{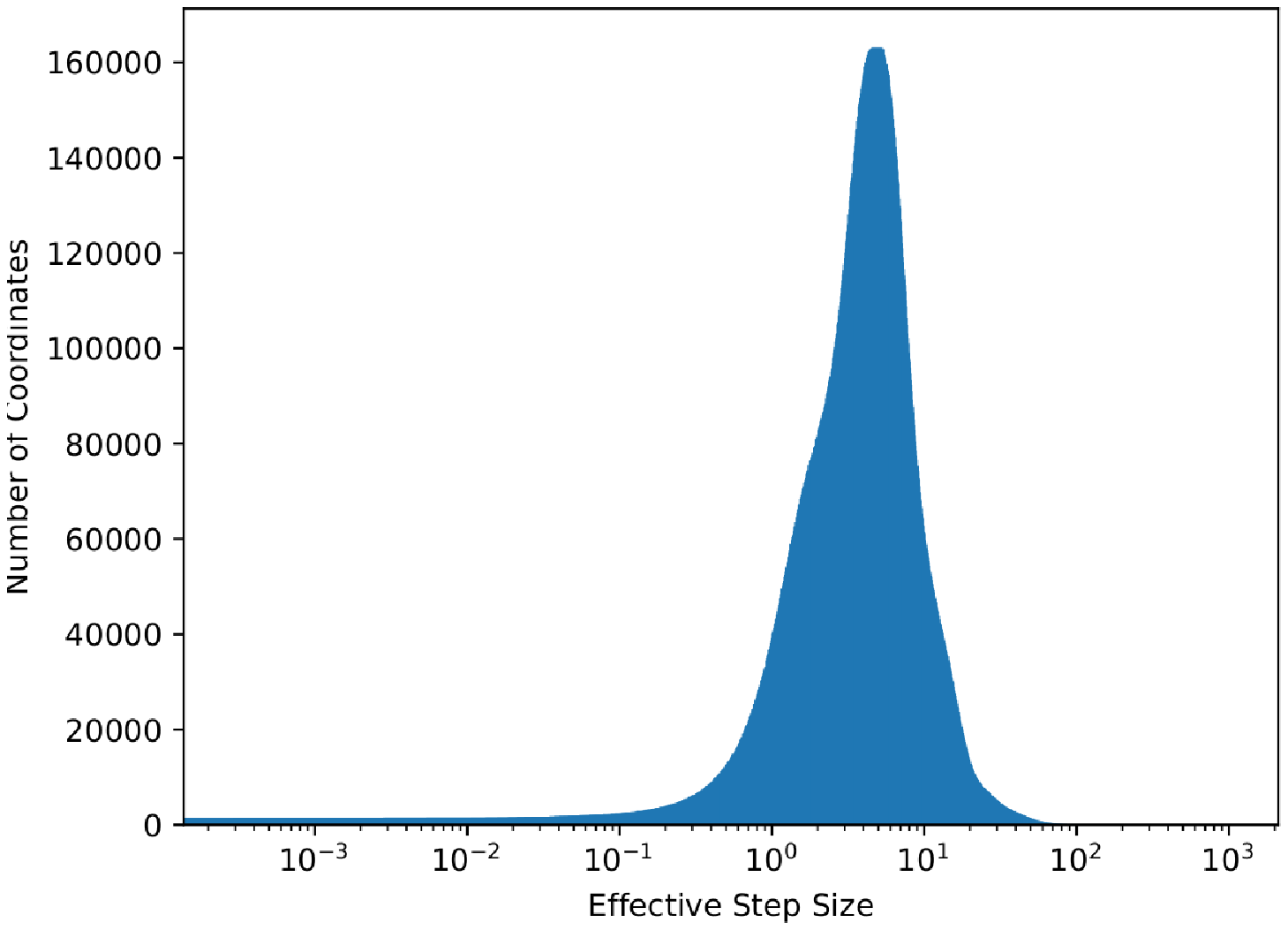}
            \includegraphics[width=0.3\textwidth]{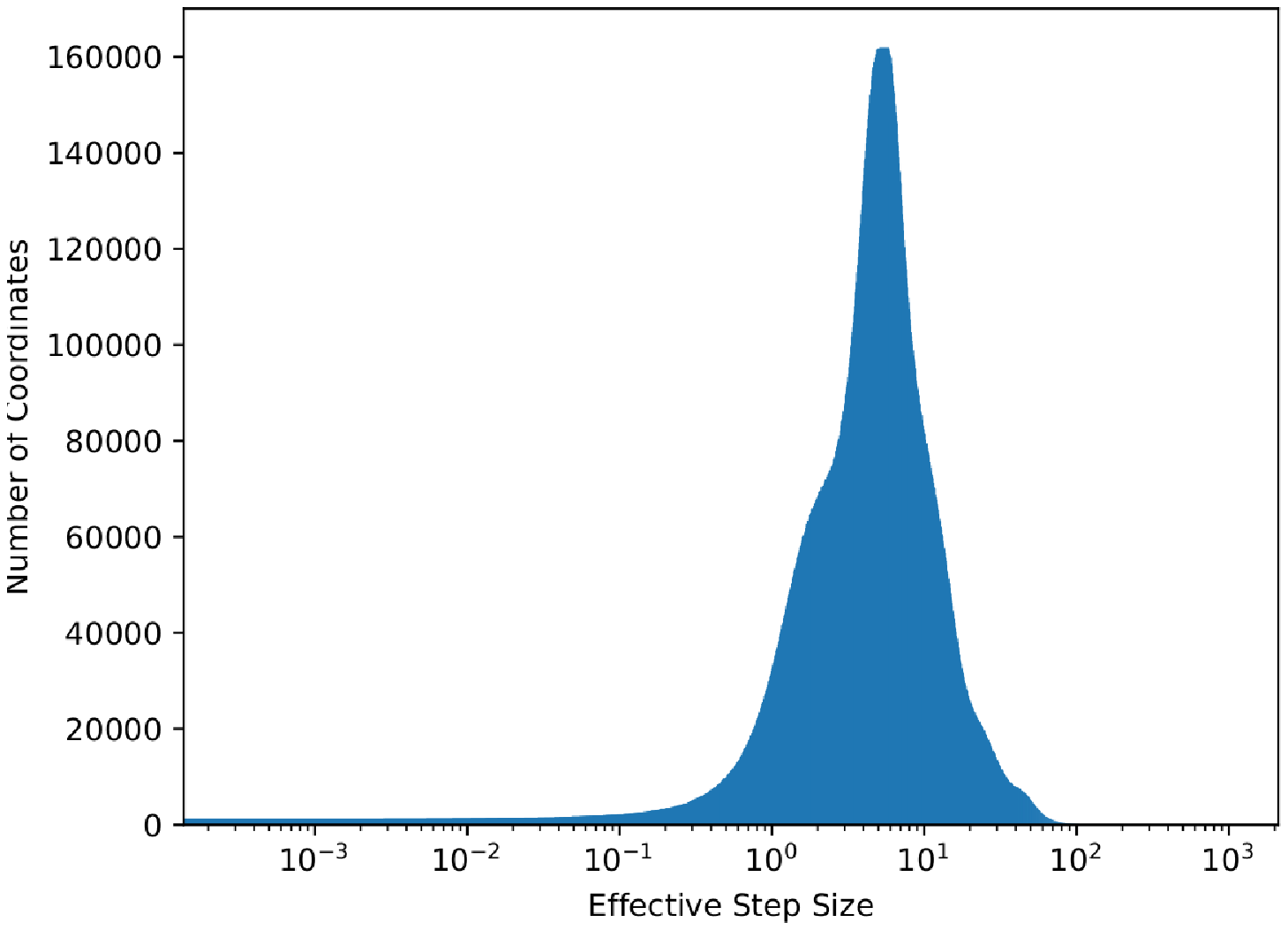}
            \includegraphics[width=0.3\textwidth]{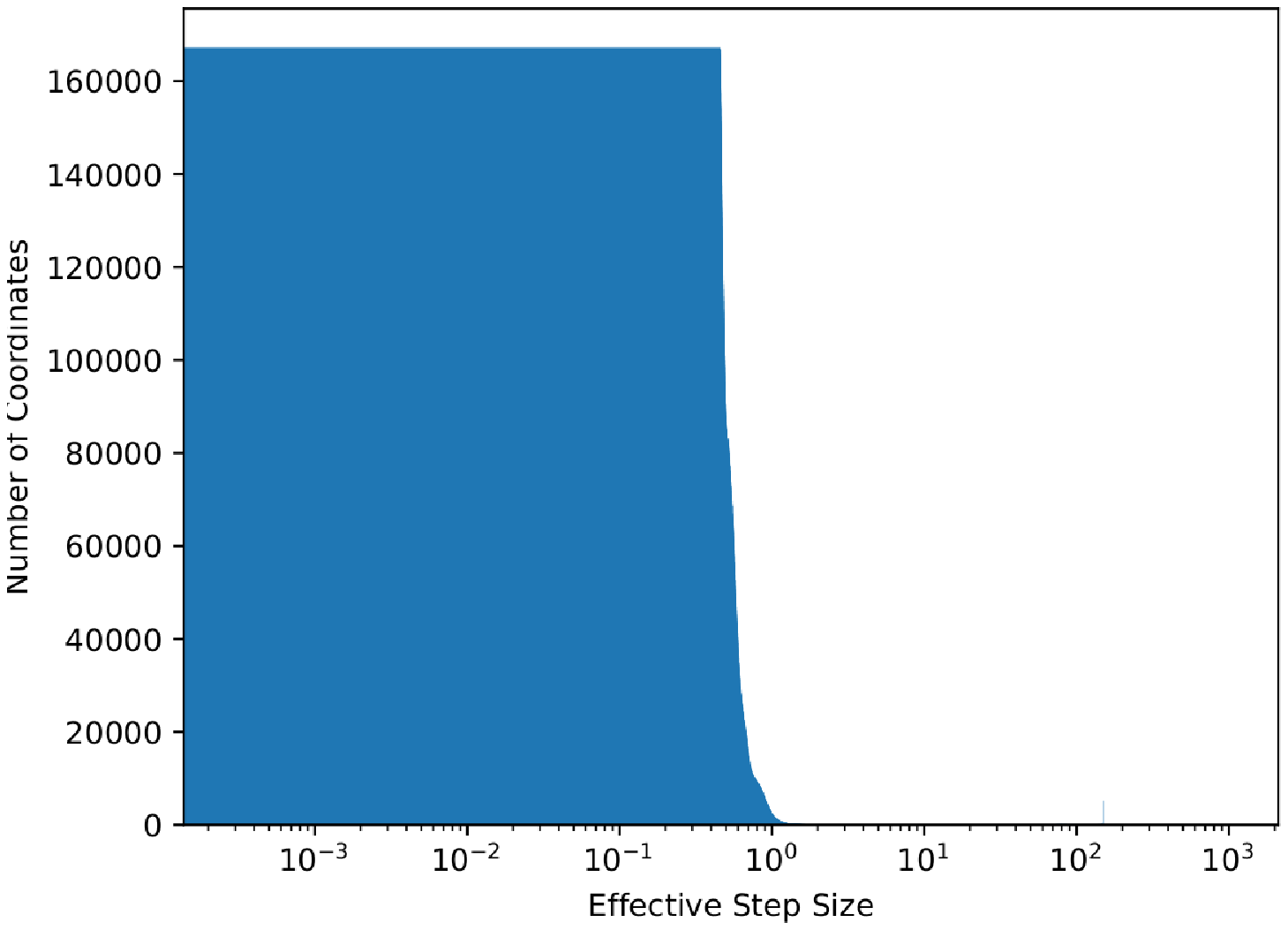}
        \end{center}
        \caption{\small Distribution of effective step size of AdaBound and MAdam at iteration 10000, 30000 and 60000 on IWSLT'14. Red lines indicate the clipping range of AdaBound. On the top/bottom are results of AdaBound/MAdam with learning rates 5e-4/1.25e-3. }
         \label{fig:adabound}
\end{figure}
Since SGD often performs much worse than Adam on transformers, and \adabound~transitions into SGD asymptotically, it is reasonable to believe that AdaBound would not converge well on transformers.
We did experiments on the IWSLT'14 dataset to evaluate \adabound~on Transformers.
AdaBound clips the effective step size to be within $0.1-\frac{0.1}{\gamma t+1}$ and $0.1-\frac{0.1}{\gamma t}$, and recommends setting $\gamma = 1-\beta_2=10^{-3}$. 
In practice, this setting only gives a $<24$ test BLEU on IWSLT'14. 
To explore the full potential of AdaBound, we tried $\gamma \in \{ 10^{-4}, 10^{-5}, 10^{-6}, 10^{-7}, 10^{-8} \}$, and found $\gamma=10^{-8}$ to give the best BLEU 35.99 (0.04).
However, as shown in Figure~\ref{fig:adabound}, AdaBound does not effectively clip most of the coordinates even in the last iteration with $\gamma=10^{-8}$, which means \adabound~essentially degraded into Adam, yet it gives better results than those effectively doing clipping. 
By comparison, the best result of MAdam and Adam with AMSGrad is 36.07(0.07) / 35.87 (0.05), respectively.

\section{Experimental Details on the Synthetic Finite-sample Experiment}
\label{sec:a_finite_samples}
Same as \cite{chen2018convergence}, we use constant learning rates $\eta$ in every step, and set $\alpha=0, \beta=0.9$ for \adam~and \amsgrad. 
For \madam, we set $\alpha=0, (\underline{\beta}, \bar{\beta})=(0.5,1)$. 
\adam~never converged for a variety of $\eta$ we tried within $[10^{-4}, 1]$, consistent with~\cite{chen2018convergence}.
Generally, a larger $\eta$ gives faster convergence for both \amsgrad~and \madam. 
For reproducibility, we repeat experiments 100 times with the same settings, and choose the $\eta$ for \amsgrad~and \madam~where the solution $|\theta^*|<0.1$ every time.
$\eta=1.2$ satisfies this requirement for \madam, but \amsgrad~only satisfied it 1\% of the times for $\eta=1.2$ and 65\% of the times for $\eta=0.9$. 
$\eta=0.8$ is the largest $\eta$ we find for \amsgrad~to achieve $100\%$ satisfaction. 
Therefore, we use $\eta=0.8$ for both \adam~and \amsgrad.

\section{Details and Additional Experimental Results on the Noisy Quadratic Model}\label{appendix:nqm}
\begin{figure}[h]
\centering
\includegraphics[width=.24\linewidth]{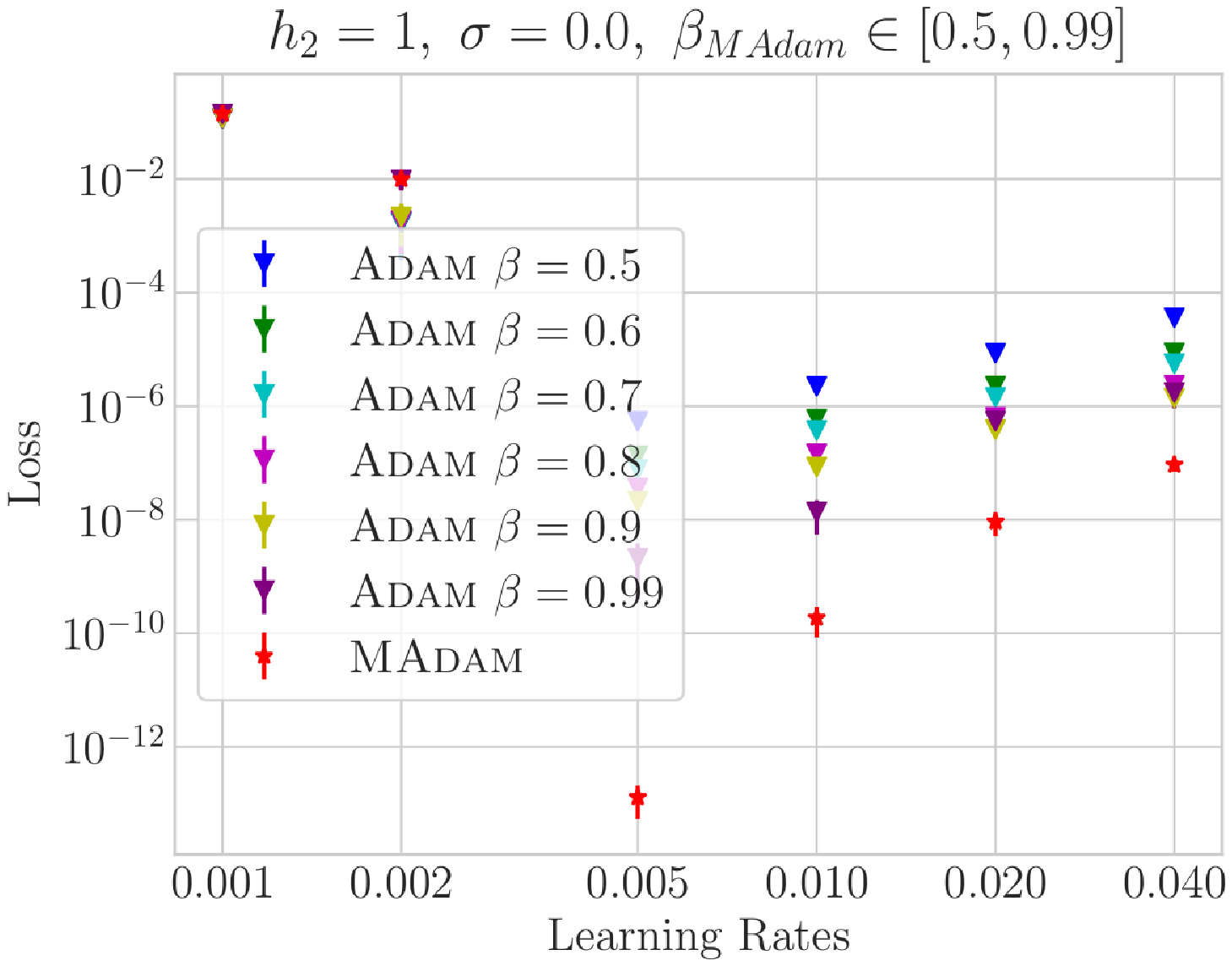}
\includegraphics[width=.24\linewidth]{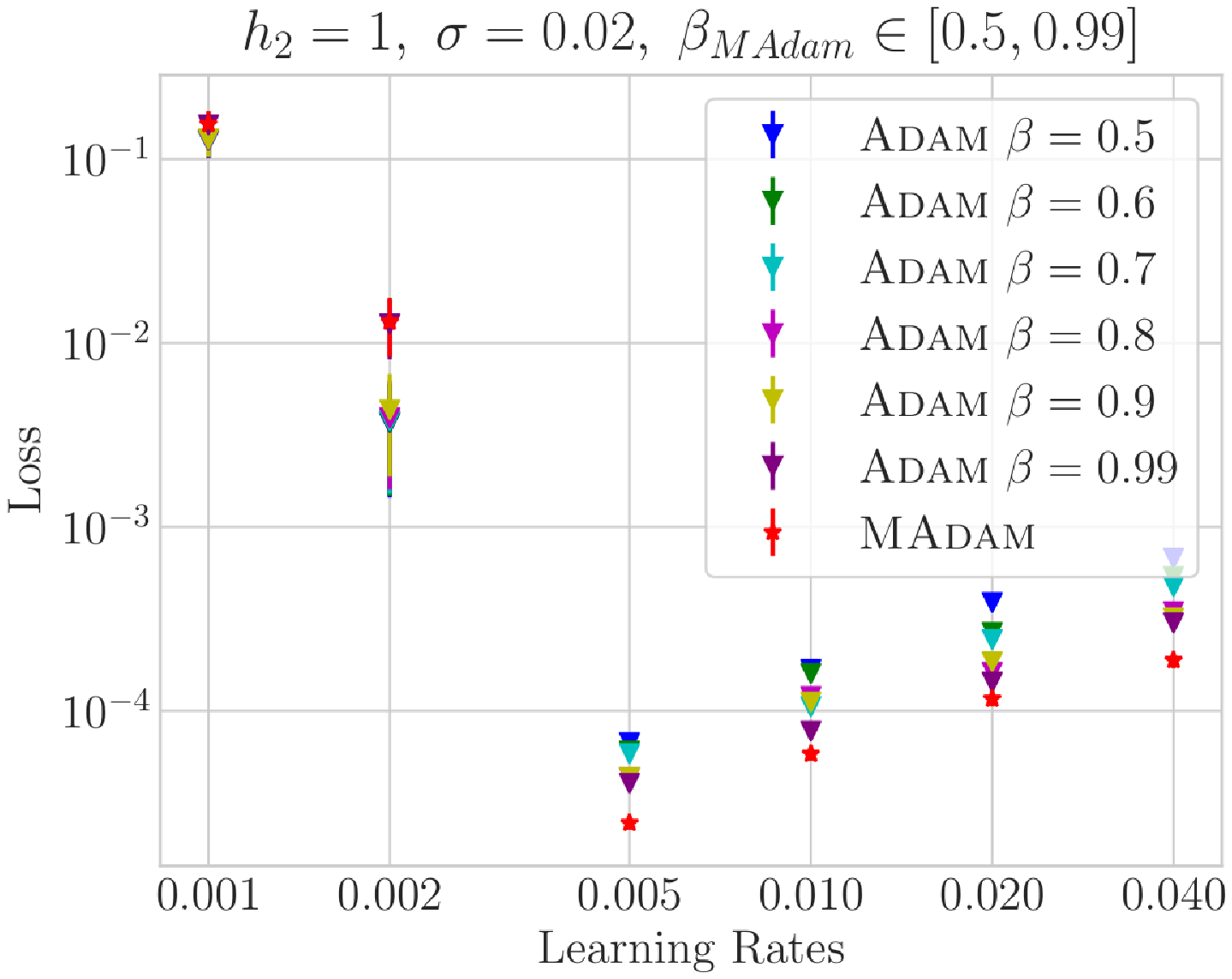}
\includegraphics[width=.24\linewidth]{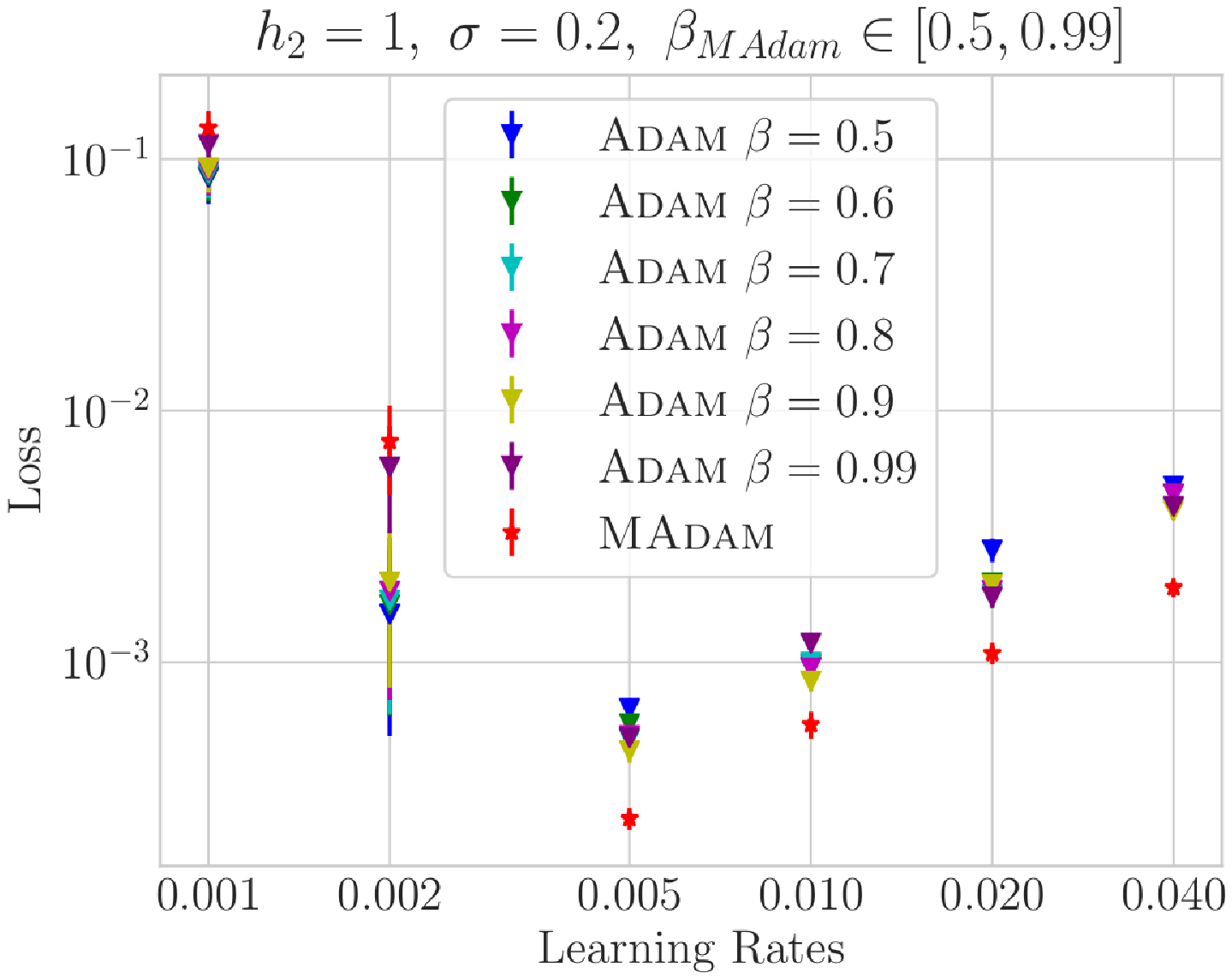}
\includegraphics[width=.24\linewidth]{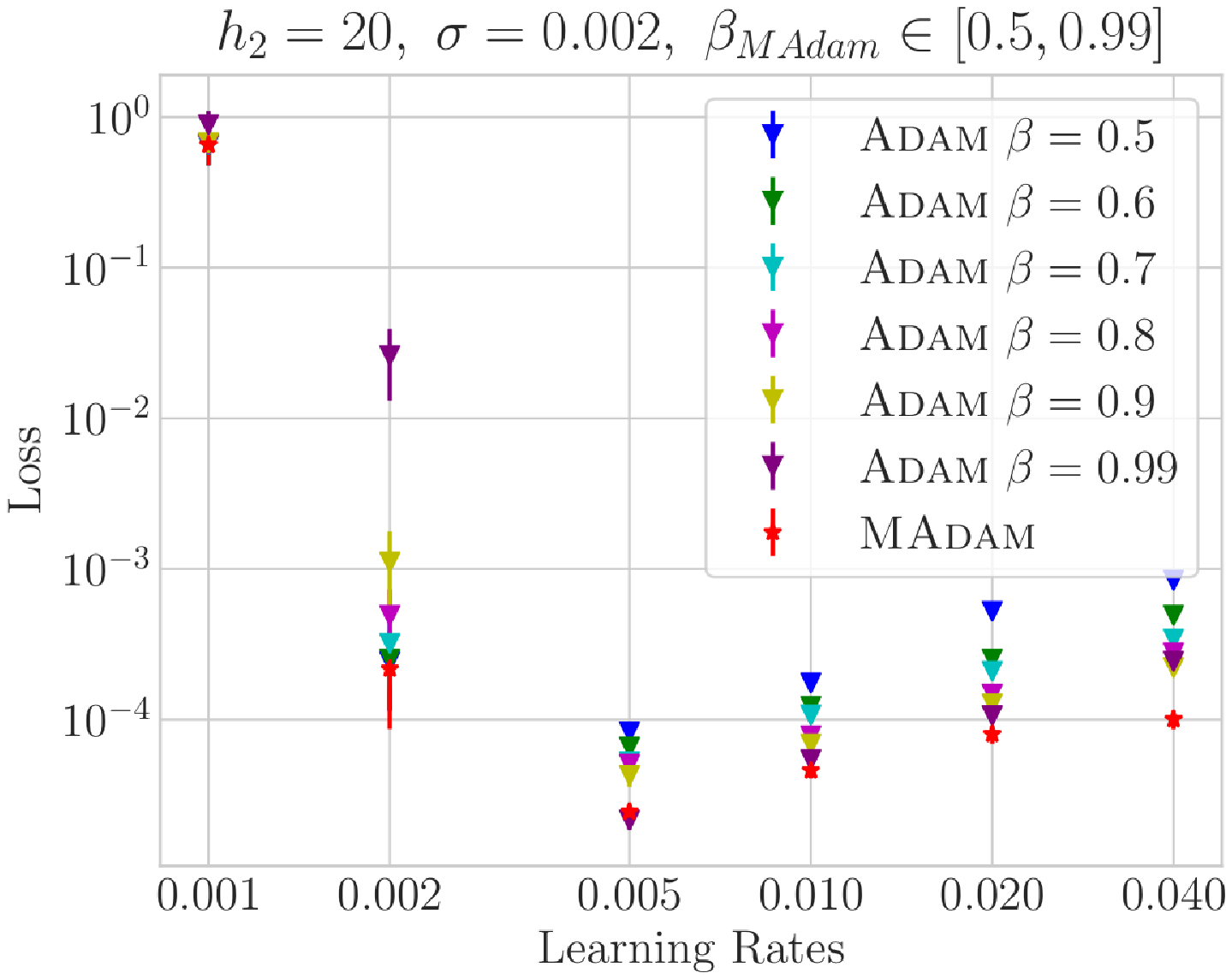}
\includegraphics[width=.24\linewidth]{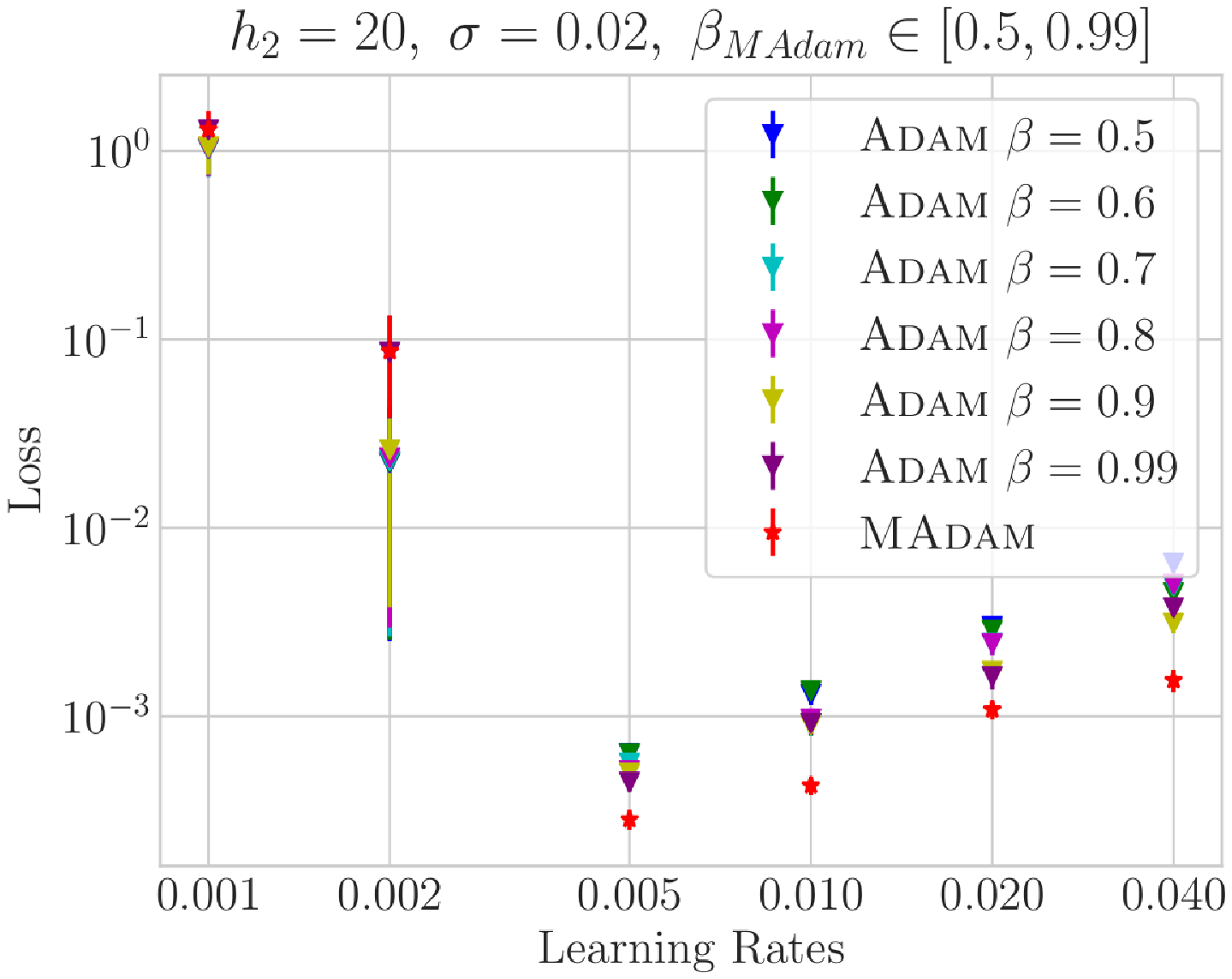}
\includegraphics[width=.24\linewidth]{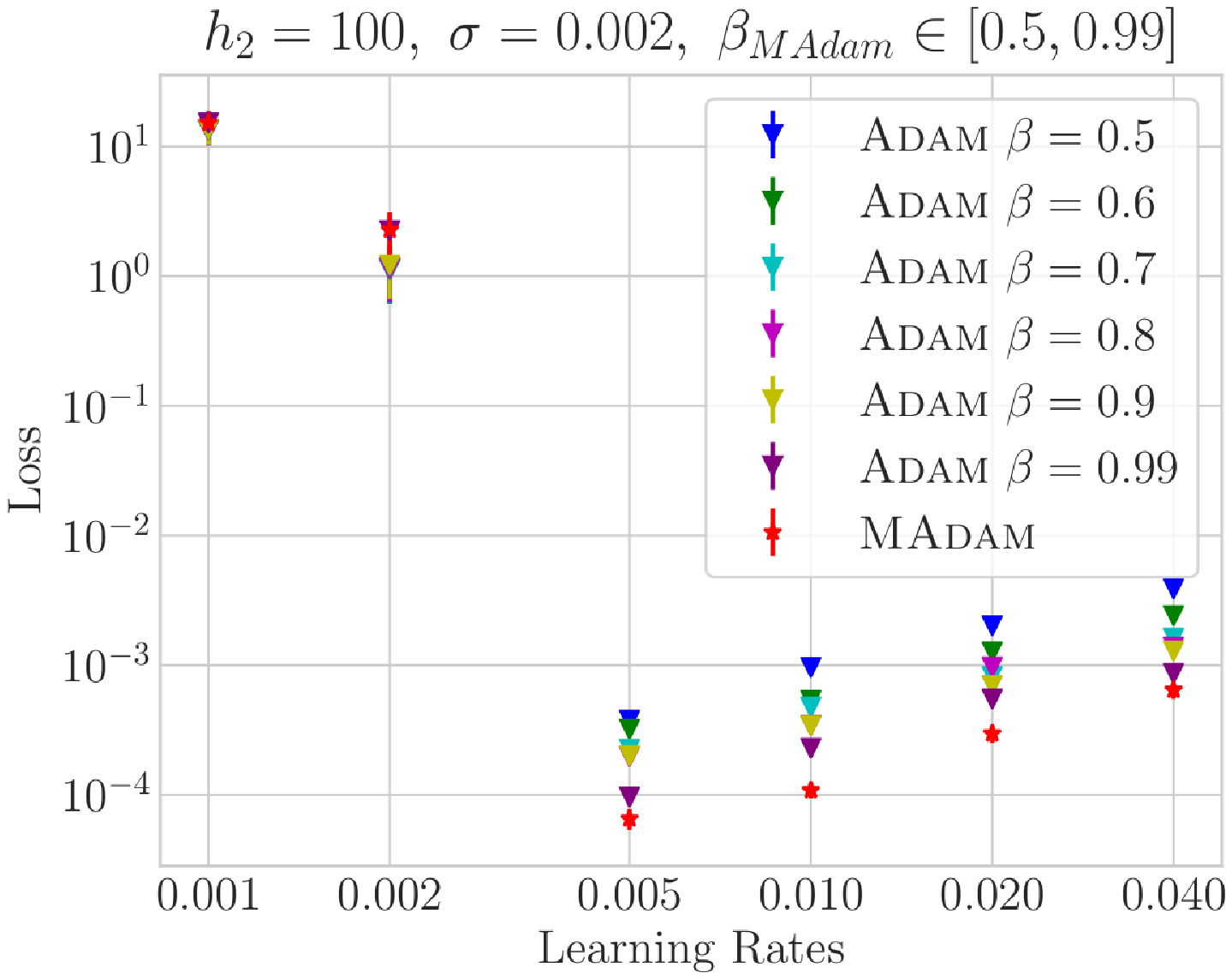}
\includegraphics[width=.24\linewidth]{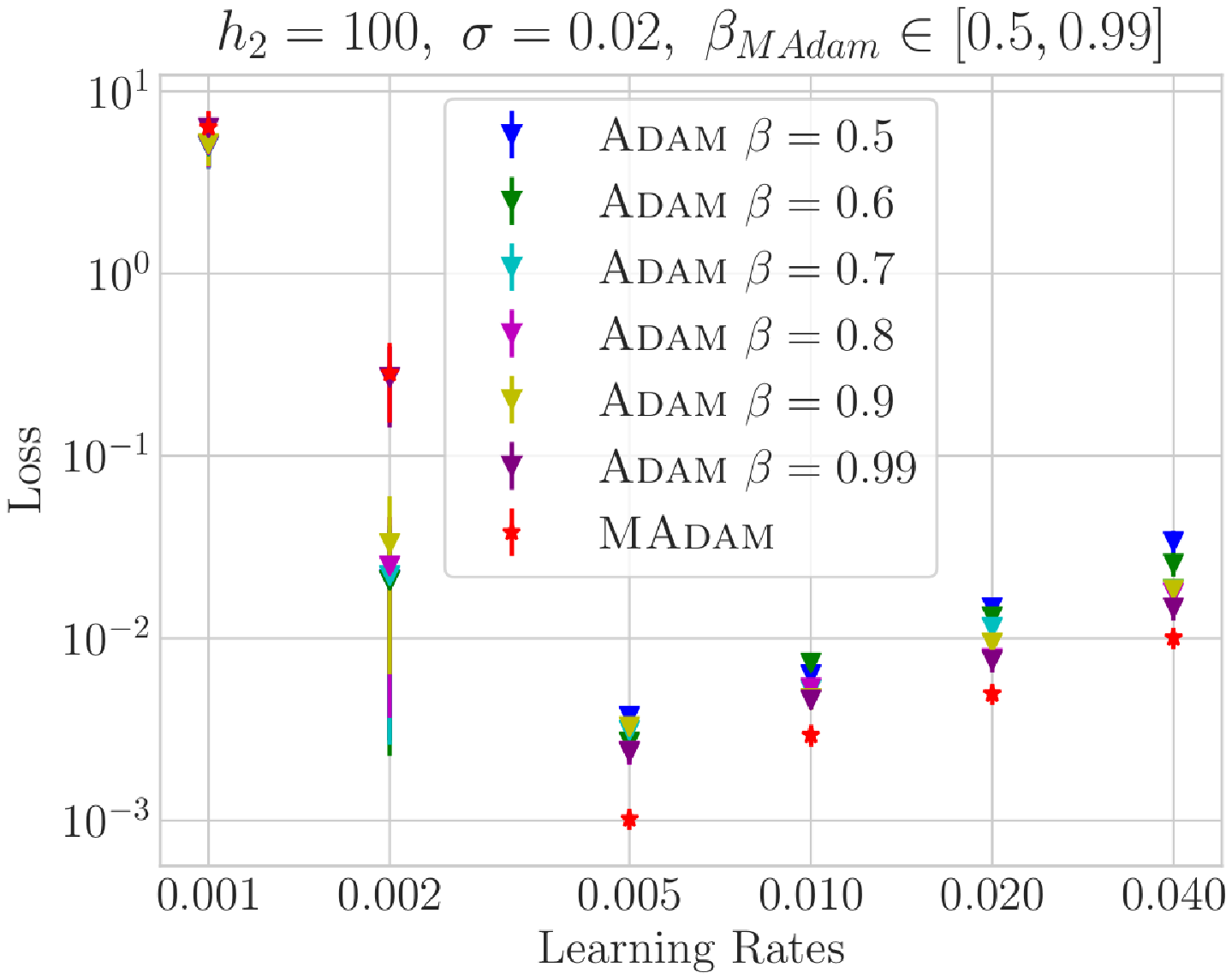}
\includegraphics[width=.24\linewidth]{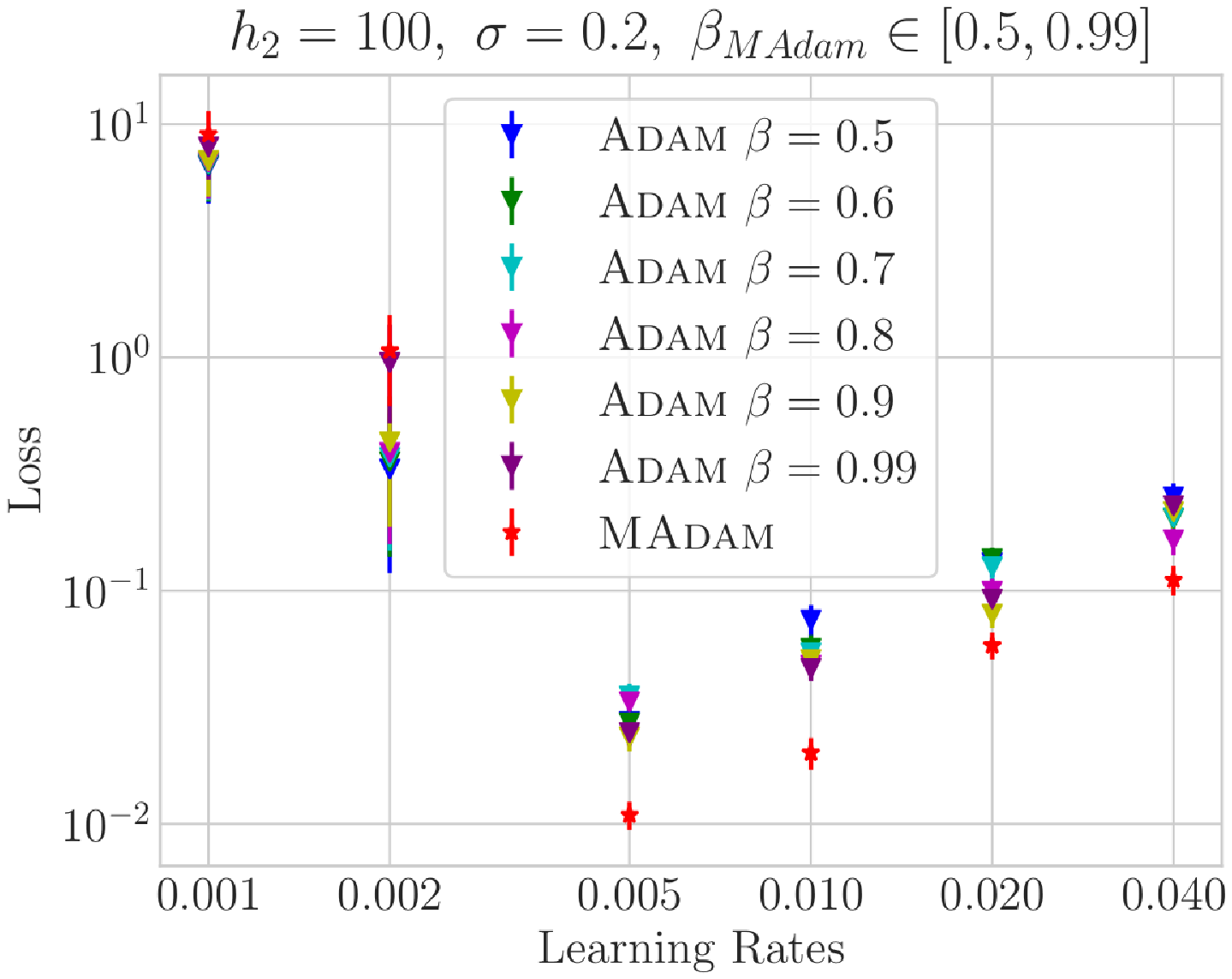}
\caption{\small More results on the Noisy Quadratic Model.}
\label{fig:nqm_appendix}
\end{figure}
\paragraph{Details of experimental settings} 
We set $\underline{\beta}=0.5, \bar{\beta}=0.99$ for \madam, and for fair comparison, we do a grid search for \adam~with $\beta\in [0.5, 0.6, 0.7, 0.8, 0.9, 0.99]$, and only report the results with the best $\beta$. 
We repeat the experiments 100 times under each setting, where we select a random initialization of $\theta\sim \gN(0, I)$ each time, and run \madam~and \adam~with different hyper-parameters from this random initialization.
Each run takes 1000 iterations by default.

\paragraph{Additional Results}
We give more results comparing \adam~and \madam~on the Noisy Quadratic Model. 
The results are shown in Figure~\ref{fig:nqm_appendix}.
Generally, the best result of \madam~has a more significant margin when $h_2$ and $\sigma$ are higher, i.e., the improvement is more significant when the problem is worse conditioned and the noise level is higher.
Note that for each trial, we start both algorithms from the same random initialization. 

\section{Details of Experiments on Image Classification}\label{appendix:image}


\paragraph{Model, larning rate schedules and data augmentations}
On CIFAR10 and CIFAR100, the ResNet18 comes from a public repository,\footnote{https://github.com/kuangliu/pytorch-cifar} which has a base width of 64 by default.
We use random cropping (4-pixel zero paddings on each side) and random horizontal flip as the data augmentations. 
Instead of using the multi-step schedule, we find the cosine learning rate schedule to yield better results for both SGD and adaptive methods. 
Therefore, we use the cosine learning rate schedule and set a final learning rate of 2e-6 in all cases.
On ImageNet, we use random resized crop and random horizontal flip for data augmentation.
For the multi-step learning rate schedule, multiply the learning rate by 0.1 every 30 epochs, and train a total of 90 epochs, with a batch size of 256. 

\paragraph{Hyperparameters of CIFAR10}
For each optimizer, we do a grid search over the learning rate and weight decay for the best hyperparameters. 
For \adam~and \laprop, we set $\beta=0.999$.  
For \madam~and \lamadam, we set $\underline{\beta}=0.5$ and $\bar{\beta}=0.999$ in all cases. 
Except for SGD, we tried learning rates from \{5e-4, 1e-3, 2e-3, 3e-3, 4e-3, 6e-3, 8e-3\} and weight decay from \{0.025, 0.05, 0.1, 0.2, 0.4, 0.8, 1 \}. 
The best learning rate and weight decay for \adam, \laprop, \madam~and \lamadam~are (3e-3, 0.2), (1e-3, 0.4), (8e-3, 0.05) and (6e-3, 0.05) respectively. 
As to SGD, we tried learning rates from \{3e-2, 5e-2, 1e-1, 2e-1, 3e-1\} and weight decays from \{1e-4, 3e-4, 5e-4, 1e-3, 2e-3\}, and the best result was achieved with learning rate 2e-1 and weight decay 3e-4.

\paragraph{Hyperparameters for CIFAR100}
We use the same grid search configurations as for CIFAR10. 
The best learning rate and weight decay for \adam, \laprop, \madam~and \lamadam~are (2e-3, 0.4), (5e-4, 1), (4e-3, 0.2) and (3e-3, 0.2) respectively. For SGD, the best learning rate and weight decay are 3e-2 and 2e-3 respectively.

\paragraph{Hyperparameters for ImageNet}
Due to the heavy workload and the time limit, we were not able to accomplish 4 runs for each hyperparameter in ImageNet, so we report the best results for each optimizer in Table~\ref{tab:image}, except for the result of \adam~, which was copied from~\cite{liu2019radam} but uses the same hyperparameters except for the learning rate and weight decay. 
For \laprop, \madam~and \lamadam, we choose learning rates from \{1e-3, 2e-3, 3e-3, 4e-3, 5e-3, 6e-3, 8e-3\} and weight decay from \{0.003, 0.006, 0.01, 0.012, 0.02, 0.03\}, and found the best combinations for \laprop, \madam~and \lamadam~are (2e-3, 0.03), (5e-3, 0.012) and (6e-3, 0.012).
For SGD, we choose learning rate from \{0.05, 0.1, 0.2\} and weight decay from \{5e-5, 7e-5, 1e-4\}, and found the best combination to be (0.1, 7e-5).

\section{Additional Experimental Results and Details on Machine Translation}\label{appendix:nmt}
\begin{figure}[t!]
\centering
\includegraphics[width=.3\linewidth]{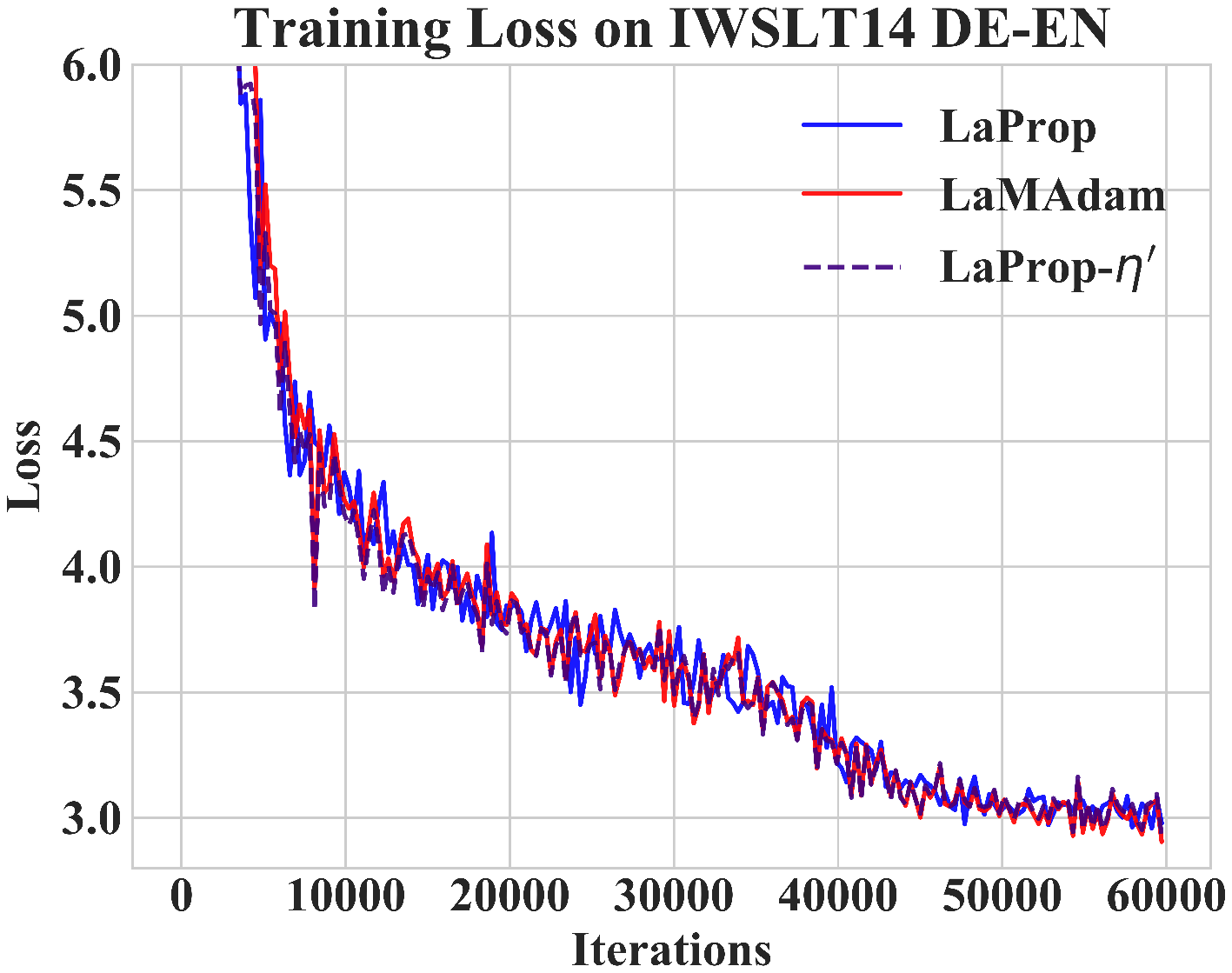}
\includegraphics[width=.3\linewidth]{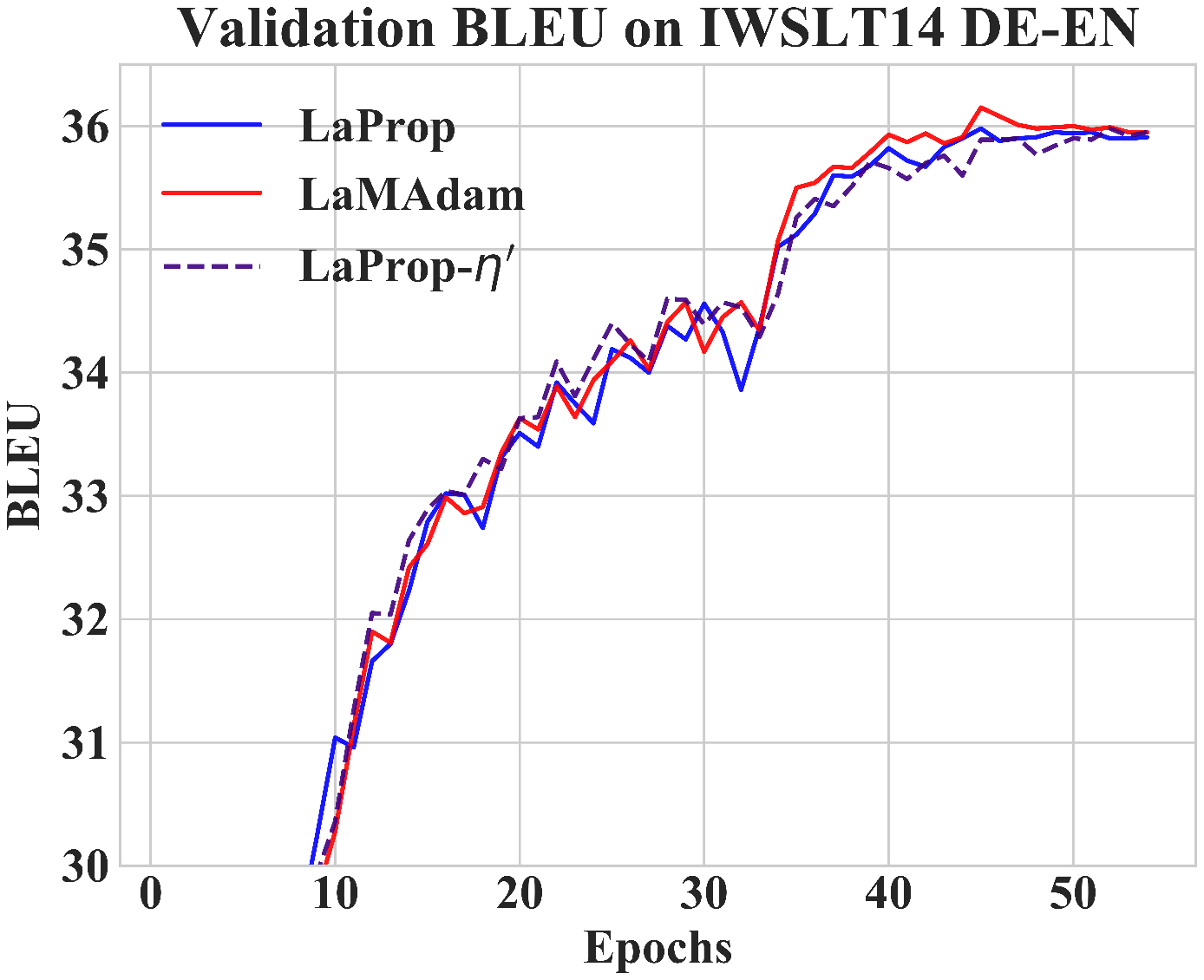}
\includegraphics[width=.3\linewidth]{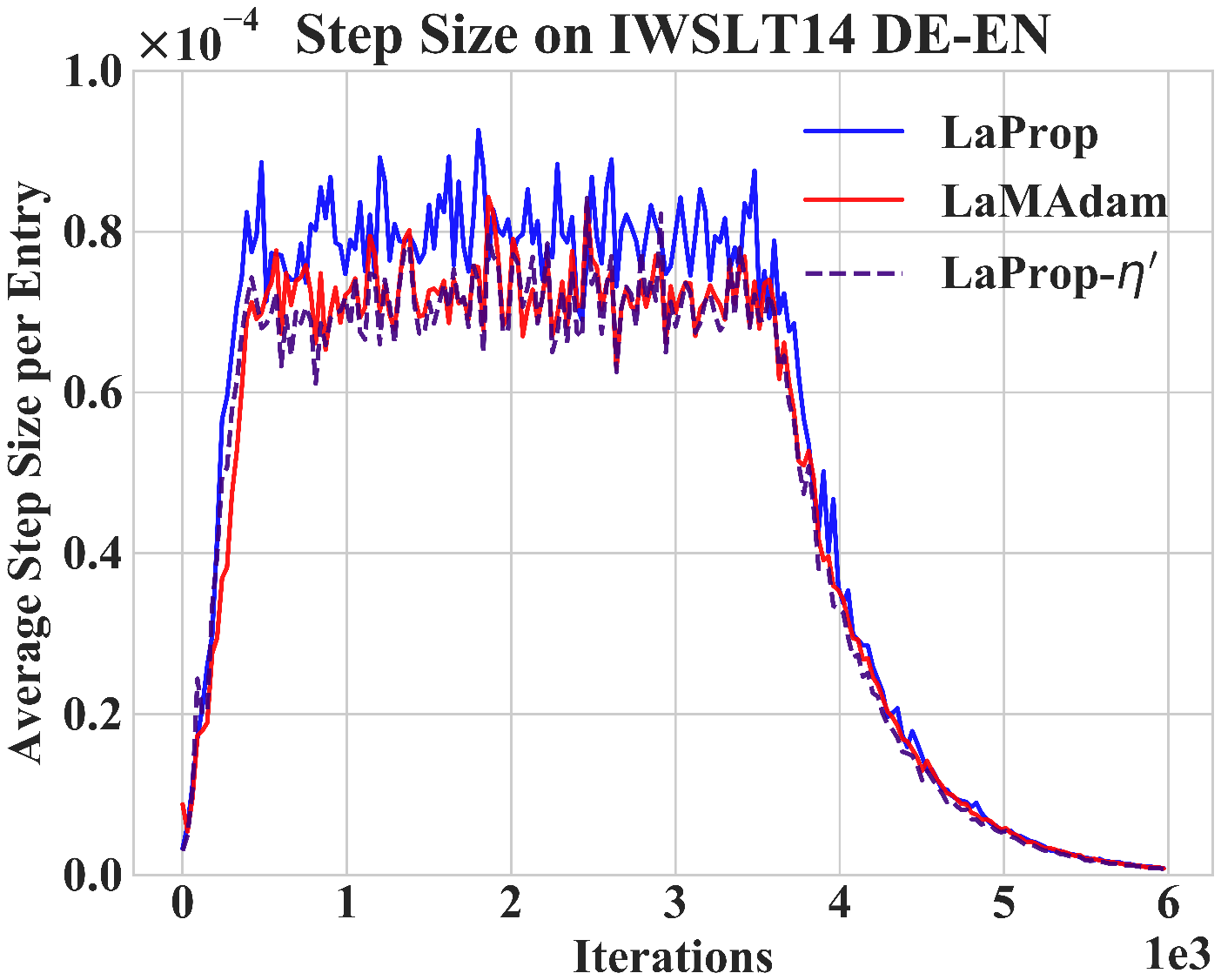}
\vspace{-1mm}
\caption{\small Training loss, validation BLEU and average step size on IWSLT'14 DE-EN, trained with $\eta$=5e-4, $\lambda$=1e-2, $\beta$=0.999 for \laprop~and $\eta$=1.5e-3, $\lambda$=1e-2, $\underline{\beta}$=0.5, $\bar{\beta}$=0.999 for \lamadam, and $\eta$=4.375e-4, $\lambda$=1e-2, $\beta$=0.999 for \laprop-$\eta'$.}
\label{fig:nmt_convergence}
\vspace{-3mm}
\end{figure}
\paragraph{Additional Experimental Results and analysis}
In Figure \ref{fig:nmt_convergence}, we plot the training loss, validation BLEU and average step size on IWSLT'14 DE-EN.
Although the average update size of \lamadam~is smaller even when using 3 times higher learning rate than \adam, \lamadam~shows slightly better convergence on the training set and better validation BLEU.
This may be explained by the heavy-tailed distribution of the gradient in the process of training transformers from scratch~\cite{zhang2019longtail}. Smaller step sizes mitigate the effect of extreme gradient values on the model's performance.
It is worth mentioning that \laprop~diverges using the large learning rate 1.5e-3. 
Further, we find \laprop~is unable to produce the same result as \lamadam~even when their update sizes are similar. 
\laprop~produces a similar step size curve as \lamadam~with learning rate 4.375e-4, but the performance is weaker than \lamadam.
\lamadam~uses the maximum variation rule to select the adaptive learning rate for each dimension, creating benefit that is not achievable by simply scaling the base learning rate $\eta$. 

\paragraph{Hyperparameters for IWSLT'14} 
The transformer we use has 512-dimensional word embeddings and 6 Transformer blocks with 4 attention heads and 1024 FFN dimensions for the encoder/decoder, which is refered to as \texttt{transformer\_iwslt\_de\_en} in fairseq.
We do a grid search for the learning rate and weight decay for both optimizers. 
We tried $\eta$ from \{2.5e-4, 5e-4, 1e-3, 1.5e-3, 2e-3\}, and weight decay from \{0.0001, 0.001, 0.01, 0.1\}.
The best combinations for \laprop~and \lamadam~are (5e-4, 0.01) and (1.5e-3, 0.01). 
To demonstrate the full potential of adaptive methods under constant learning rates, we use the tri-stage learning rate schedule~\cite{park2019tristage}, linearly increase the learning rate from $0.01\eta$ to the full learning rate $\eta$ in 4k iterations, hold it at $\eta$ for 32k iterations, and exponentially decay it to $0.01\eta$ in 24k iterations. 
For \laprop, we tried $\beta$ from \{0.98, 0.99, 0.997, 0.999\}. 
We found 0.999 to work the best and used it for all the grid search experiments. 
For \lamadam, we set $\underline{\beta}=0.5$, $\bar{\beta}=0.999$.
For other hyperparameters, we use the default setting in the fairseq example, which sets dropout probability to 0.3, uses label smoothed cross entropy loss with a smoothing coefficient 0.1, and shares the input and output token embedding parameters.

\paragraph{Hyperparameters for WMT'16} 
The Transformer we use has 1024-dimensional word embeddings, 6 transformer blocks with 16 attention heads and 4096 FFN dimensions for the encoder/decoder, and is refered to as \texttt{transformer\_vaswani\_wmt\_en\_de\_big} in fairseq. 
The default implementation from fairseq did not use weight decay, so we also ignore weight decay in all experiments. 
The learning rate schedule takes the first 4k steps to linearly increase the learning rate to its maximum value.
For \laprop, we found $\beta=0.98$ to give the best results, and we set $\underline{\beta}=0.95, \bar{\beta}=0.98$ in all experiments.
This takes around 8 hours on 16 V100 GPUs each run.
For grid search, we tried $\eta$ from \{5e-4, 1e-3, 1.5e-3, 2e-3\}, and found 1e-3 and 1.5e-3 to work the best for \laprop~and \lamadam~respectively.
Other hyperparameters are the defaults of the corresponding fairseq example, which uses a dropout probability of 0.3, the label smoothed cross entropy loss with a smoothing coefficient 0.1, and shares all embedding parameters. 

\section{Additional Details of Experiments on the GLUE benchmark}\label{appendix:glue}

The GLUE benchmark is a collection of 9 natural language understanding tasks, 
namely Corpus of Linguistic Acceptability (CoLA; \cite{warstadt2018neural}), Stanford Sentiment Treebank (SST; \cite{socher2013recursive}), Microsoft
Research Paraphrase Corpus (MRPC; \cite{dolan2005automatically}), Semantic Textual Similarity Benchmark (STS; \cite{agirre2007semantic}), Quora Question Pairs (QQP; \cite{qqp2016url}), Multi-Genre NLI (MNLI; \cite{williams2018broad}), Question NLI (QNLI; \cite{rajpurkar2016squad}), Recognizing Textual Entailment (RTE; \cite{dagan2006pascal}; \cite{bar2006second}; \cite{giampiccolo2007third}; \cite{bentivogli2009fifth}) and Winograd NLI (WNLI; \cite{levesque2011winograd}).

It is reported in~\cite{liu2019roberta} that \adam~is sensitive to the choice of $\epsilon$ on GLUE. 
Following their settings, we set $\epsilon=1e-6$ for both \adam~and \madam. 
For \laprop~and \lamadam, however, we always set $\epsilon=1e-15$, like all other experiments in this paper, which is consistent with the observation in~\cite{ziyin2020laprop} that \laprop~is robust to the choice of $\epsilon$.
We set $\beta=0.98$ for \adam~and \laprop, and $\underline{\beta}=0.5, \bar{\beta}=0.98$ for \laprop~and \lamadam.
All other hyperparameters are set to the same as the example in fairseq.\footnote{https://github.com/pytorch/fairseq/blob/master/examples/roberta/README.glue.md}
For each task, we do a grid search over the learning rate and weight decay, which are chosen from \{5e-6, 1e-5, 2e-5, 4e-5, 5e-5, 6e-5\} and \{0.025, 0.05, 0.1, 0.2\} respectively.
We list the best combinations for \adam, \madam, \laprop~and \lamadam~on each task as below:
\begin{itemize}
    \item[\textbf{MNLI}:] (1e-5, 0.1), (1e-5, 0.1), (4e-5, 0.025), (4e-5, 0.025).
     \item[\textbf{QQP}:] (1e-5, 0.1), (1e-5, 0.1), (4e-5, 0.025), (4e-5, 0.025).
    \item[\textbf{QNLI}:] (1e-5, 0.1), (1e-5, 0.1), (4e-5, 0.05), (4e-5, 0.05).
    \item[\textbf{SST-2}:] (1e-5, 0.1), (1e-5, 0.1), (4e-5, 0.1), (4e-5, 0.1).
    \item[\textbf{RTE}:] (2e-5, 0.1), (2e-5, 0.1), (6e-5, 0.1), (6e-5, 0.1).
    \item[\textbf{MRPC}:] (1e-5, 0.1), (1e-5, 0.1), (6e-5, 0.1), (6e-5, 0.1).
    \item[\textbf{STS-B}:] (2e-5, 0.1), (2e-5, 0.1), (4e-5, 0.5), (4e-5, 0.5).
    \item[\textbf{CoLA}:] (2e-5, 0.1), (2e-5, 0.1), (6e-5, 0.5), (6e-5, 0.5).
\end{itemize}

\end{document}